\documentclass{lmcs}
\pdfoutput=1
\usepackage[utf8]{inputenc}

\usepackage{lastpage}
\lmcsdoi{21}{3}{14}
\lmcsheading{}{\pageref{LastPage}}{}{}%
{Jun.~16,~2023}{Aug.~05,~2025}{}

\usepackage{hyperref}
\theoremstyle{plain} 

\usepackage{multirow}

\usepackage{amssymb,amsmath}
\usepackage{wrapfig}
\usepackage{doi}
\usepackage{centernot}

\newcommand\cA{\ensuremath{\mathcal{A}}\xspace}
\newcommand\cB{\ensuremath{\mathcal{B}}\xspace}
\newcommand\cC{\ensuremath{\mathcal{C}}\xspace}
\newcommand\cD{\ensuremath{\mathcal{D}}\xspace}
\newcommand\cE{\ensuremath{\mathcal{E}}\xspace}

\newcommand\cI{\ensuremath{\mathcal{I}}\xspace}
\newcommand\cJ{\ensuremath{\mathcal{J}}\xspace}

\newcommand\cK{\ensuremath{\mathcal{K}}\xspace}
\newcommand\cL{\ensuremath{\mathcal{L}}\xspace}

\newcommand\cR{\ensuremath{\mathcal{R}}\xspace}

\newcommand\cT{\ensuremath{\mathcal{T}}\xspace}

\newcommand\cW{\ensuremath{\mathcal{W}}\xspace}

\usepackage{xspace}
\newcommand{\alc}{$\mathcal{ALC}$\xspace}

\newcommand{\shiq}{$\mathcal{SHIQ}$\xspace}

\newcommand{\sroiqd}{$\mathcal{SROIQ}$(\textbf{D})\xspace}

\newcommand{\dlf}[1]{%
	\ifx\@currenvir\@math%
	{\normalfont\mathsf{#1}}%
	\else%
	{\normalfont\textsf{#1}}%
	\fi%
}%

\begin{document}
	
	\title[Querying Inconsistent Uncertain Knowledge Bases]{Exploiting Uncertainty for Querying Inconsistent Description Logics Knowledge Bases}
	
	\author[R.~Zese]{Riccardo Zese\lmcsorcid{0000-0001-8352-6304}}[a]
	\author[E.~Lamma]{Evelina Lamma\lmcsorcid{0000-0003-2747-4292}}[b]
	\author[F.~Riguzzi]{Fabrizio Riguzzi\lmcsorcid{0000-0003-1654-9703}}[c]
	
	\address{Dip. di Scienze Chimiche, Farmaceutiche e Agrarie, Universit\`a di Ferrara, Via Luigi Borsari 46, 44121, Ferrara, Italy}	
	\email{riccardo.zese@unife.it}  
	
	\address{Dip. di Ingegneria, Universit\`a di Ferrara, Via Saragat 1, 44122, Ferrara, Italy}	
	\email{evelina.lamma@unife.it}  
	
	\address{Dip. di Matematica e Informatica, Universit\`a di Ferrara, Via Macchiavelli 30, 44122, Ferrara, Italy}	
	\email{fabrizio.riguzzi@unife.it}  

	\begin{abstract}
		The necessity to manage inconsistency in Description Logics Knowledge Bases~(KBs) has come to the fore with the increasing importance gained by the Semantic Web, where information comes from different sources that constantly change their content and may contain contradictory descriptions when considered either alone or together.
		Classical reasoning algorithms do not handle inconsistent KBs, forcing the debugging of the KB in order to remove the inconsistency. In this paper, we exploit an existing probabilistic semantics called DISPONTE to overcome this problem and allow queries also in case of inconsistent KBs. 
		We implemented our approach in the reasoners TRILL and BUNDLE and empirically tested the validity of our proposal. Moreover, we formally compare the presented approach to that of the repair semantics, one of the most established semantics when considering DL reasoning tasks.
	\end{abstract}
	
	\maketitle
	
	\section{Introduction}
	\label{sec:intro}
	In the Semantic Web, one of the main goals is to create comprehensive Knowledge Bases (KBs) by combining contributions coming from different sources. This may easily lead to contradictions, as the different sources my represent different points of view.
	
	A classic example is given by the \textit{flying penguin} problem, where a KB defines that all birds fly and that penguins are birds, but they are not able to fly. This is clearly a contradiction and so, if the KB asserts that \textit{pingu} is a penguin, the KB is inconsistent because \textit{pingu} belongs to both the concept ``fly'' and to its complement. With standard reasoning techniques, when the KB is inconsistent, inference is trivial as anything is entailed by an inconsistent KB. For this reason, systems implementing such techniques do not allow the execution of queries when the KB is inconsistent or allow only the identification of axioms causing the inconsistency.
	In non-monotonic reasoning this problem has been solved by considering a unique KB where non-monotonic knowledge representation is adopted, often with negative literals for coping with exception or abnormalities as penguins in the example above. 
	
	Other approaches to manage inconsistent pieces of information use monotonic logics, such as Description Logics (DLs), and  the so-called \emph{repairs}, consistent subsets of axioms of the KB built when the query is asked. Given the set of repairs of a KB, there are different semantics that define the conditions the query must fulfil to be true~\cite{DBLP:conf/rr/LemboLRRS10,DBLP:conf/rweb/BienvenuB16}, among them, the most used are the Brave~\cite{DBLP:conf/ijcai/BienvenuR13}, AR~\cite{DBLP:conf/rr/LemboLRRS10}, and IAR~\cite{DBLP:conf/rr/LemboLRRS10} semantics. 
	
	However, despite the number of works on managing inconsistent KBs, very few proposals consider the fact that information is usually uncertain in real world domains. To fill this gap, we exploit the DISPONTE semantics~\cite{RigBelLamZes15-SW-IJ,Zese17-SSW-BK}, where the axioms of the KBs are associated to probability values, defining the degree of belief in the presence of the axiom in the KB. 
	Defining these probability values could be done by asking to a domain expert.
	Or else, one can associate a probability value to axioms depending on how much they trust the source of the information, giving more confidence, in the example above, to information coming from ornithologists than that coming from, e.g., Wikipedia. Moreover, one interesting feature of DISPONTE is that it can be applied to the underlying Semantic Web language by  exploiting annotations: the probability can be added to axioms without affecting the syntax of the language. 
	This avoids limitations on the expressivity of the languages that are used to define the KBs and compatibility problems with other KBs.

	The contributions of the paper are fourfold.
	First, we show how an existing  probabilistic semantics, DISPONTE, can be used to reason with inconsistent DL KBs by returning a value expressing how much we can  believe in the truth of the query. With respect to existing approaches, our proposal is not only able to deal with very expressive DLs, features that shares with some of the other approaches, but it also give an implemented system able to manage the entire pipeline of answering a query w.r.t. KBs modelled using very expressive DLs.
	Second, the paper discusses the changes that must be implemented to allow a reasoner exploiting the tableau algorithm to cope with this semantics and answer queries w.r.t. possibly inconsistent KBs.
	Third, despite the number of proposals for dealing with inconsistent KBs, there are few  systems that are implemented: the paper provides two different working implementations of this approach included in the reasoners BUNDLE~\cite{RigBelLamZese13-RR13b-IC,CotRigZesBelLam-IC-2018} and TRILL~\cite{DBLP:journals/tplp/ZeseCLBR19,ZesBelRig16-AMAI-IJ,DBLP:journals/ws/ZeseC21}. These are two probabilistic reasoners that answer queries under DISPONTE, the first is implemented in Java, while the second in Prolog. In particular, BUNDLE encapsulates several reasoners, such as the non-probabilistic reasoners Pellet~\cite{DBLP:journals/ws/SirinPGKK07} and Hermit~\cite{shearer2008hermit}, or the probabilistic reasoner TRILL, which can be used both inside BUNDLE and stand-alone. As we will show, the extension can be easily implemented in other reasoners in order to answer queries w.r.t. inconsistent KBs.
	Fourth, the paper compares the presented approach to that of the repair semantics, one of the most established semantics when considering DL reasoning tasks, discussing how the reasoning workflow of the approach can be adapted to also answer queries under different repair semantics. In fact, the process of computing the probability of a query is somewhat similar to the construction of the repairs, as it needs to collect consistent subsets of the axioms of the KB that make the query true.
	Fifth, these theoretical results have been implemented in BUNDLE and TRILL that are able to follow both the DISPONTE and repair semantics. With the proposed extension we can compute at the same time (i) the justification of the queries and those of the inconsistency, (ii) the DISPONTE probability of queries and (iii) their truth under Brave/AR/IAR repair semantics w.r.t. every KB described in any DL accompanied by a system based on the tableau capable of collecting justifications, in order to give to the user more insight on the truth of the queries asked.

	The paper is organized as follows. Section~\ref{sec:bg} introduces background information about Description Logics, their extension to probability by means of DISPONTE, and recalls the most common repair semantics. 
	Section~\ref{sec:query-inc-kb} discusses the proposed probabilistic semantics and the extensions needed by a reasoner to cope with it. Section~\ref{sec:comp-repair} illustrates the differences of our semantics with the repair semantics, considering the Brave, AR and IAR semantics. Section~\ref{sec:related} discusses related work.
	Section~\ref{sec:exp} shows the results of some tests of two prototype reasoners implementing the extension discussed in Section~\ref{sec:reas-alg}. The implementation of two different prototypes is also intended to demonstrate the ease with which this semantics can be adopted.  Finally, Section~\ref{sec:conc} concludes the paper.

	\section{Preliminaries}
	\label{sec:bg}

	In this section we present the background necessary for the next sections. In the first sub-section, we will briefly describe the $\mathcal{ALC}$ DL. Then we will move on to introduce the semantics DISPONTE, also describing the syntax to use to associate probability to axioms and how to use them to calculate the probability of queries. Finally, we will provide the definitions of the repair semantics.

	\subsection{Description Logics}
	\label{sec:description-logics}
	
	DLs are a family of logic-based knowledge representation formalisms which are
	of particular interest for representing ontologies in the Semantic Web. For a good
	introduction to DLs we refer to~\cite{Badeer:2008:DL:52211}.

	The different DL languages represent information by means of individuals (domain elements), concepts (unary relations that can be interpreted by sets of domain elements), and roles (binary relations that can be interpreted by sets of pairs of domain elements). They differ in how concepts and roles can be combined for defining complex concepts.
	We briefly review the \alc DL, however, the results presented in this paper can be applied to any DL.

	Let $N_I$, $N_R$, and $N_C$ be countable infinite sets of individual, role and concept names.
	\emph{Concepts} $C$ are defined as 
	\begin{align}
		C::= A\ |\ \bot\ |\ \top\ |\ (C\sqcap C)\ |\ (C\sqcup C) \ |\ \neg C\ |\ \exists R.C\ |\ \forall R.C\label{grammar}
	\end{align}
	where  $A\in N_C$, $R\in N_R$.

	A \emph{knowledge base} (KB) consists of a finite set of \textit{general concept inclusion axioms} (GCIs) $C_1\sqsubseteq C_2$ where $C_1$ and $C_2$ are concepts built according to the grammar (\ref{grammar}), called \emph{TBox}, and a finite set of \textit{concept assertions} $a:C$ and \textit{role assertions} $(a, b):R$ where $R \in N_R$, and $a,b \in N_I$, called \emph{ABox}. A knowledge base is a pair  $\cK=(\cA,\cT)$ where \cA is a ABox and \cT a TBox.
	
	The semantics of DLs is formally defined using interpretations $\cI = (\Delta^\cI , \cdot^\cI )$, where $\Delta^\cI$ is a non-empty \emph{domain}, and $\cdot^\cI$ is an \emph{interpretation function} that maps each $a\in N_I$ to an element of $\Delta ^\cI$, each $C\in N_C$ to a subset of $\Delta^\cI$, and each $R\in N_R$ to a subset of $\Delta^\cI \times \Delta^\cI$.
	The mapping $\cdot^\cI$ is extended to complex concepts as follows (where 
	$R^\cI(x) = \{y | (x, y) \in R^\cI\}$):
	\begin{footnotesize}
		$$
		\begin{array}{rcl}
			\top^\cI&=&\Delta^\cI \\
			\bot^\cI&=&\emptyset\\
			(\neg C)^\cI&=&\Delta^\cI\setminus C^\cI\\
			(C_1\sqcup C_2)^\cI&=&C_1^\cI \cup C_2^\cI\\
			(C_1\sqcap C_2)^\cI&=&C_1^\cI \cap C_2^\cI\\
			(\exists R.C)^\cI&=&\{x\in \Delta^\cI| R^\cI(x)\cap C^\cI\neq \emptyset\}\\
			(\forall R.C)^\cI&=&\{x\in \Delta^\cI| R^\cI(x)\subseteq C^\cI\}
		\end{array}
		$$
	\end{footnotesize}
	
	Given a specific interpretation $\cI$, we can determine the \emph{satisfaction} of an axiom $Q$, i.e., whether the axiom holds (is true) with respect to $\cI$.  The satisfaction, denoted by $\cI \models Q$, is defined as follows:
	\begin{enumerate}
		\item a concept inclusion axiom $C\sqsubseteq D$ is satisfied in $\cI$ iff $C^\cI \subseteq D^\cI$, 
		\item a concept assertion axiom $a : C$ is satisfied in $\cI$ iff $a^\cI \in C^\cI$, 
		\item a role assertion axiom $(a, b) : R$ is satisfied by $\cI$ iff $(a^\cI , b^\cI ) \in R^\cI$, 
	\end{enumerate}

	\noindent 
	$\cI$ \textit{satisfies} a set of axioms $\mathcal{E}$, denoted by $\cI \models \mathcal{E}$, iff $\cI \models  E$ for all $E \in  \mathcal{E}$.
	An interpretation $\cI$ \textit{satisfies} a knowledge base $\cK$, denoted $\cI \models \cK$, iff $\cI$ satisfies all the axioms contained in $\cK$. In this case we say that $\cI$ is a \emph{model} of $\cK$.
	A knowledge base $\cK$ is \textit{consistent} (or \textit{satisfiable}) iff there exists an interpretation $\cI$ that satisfies $\cK$, otherwise it is \textit{inconsistent} (or \textit{unsatisfiable}),
	i.e., $\cK\models\perp$. An axiom $Q$ is  \textit{entailed} by $\cK$, denoted $\cK \models Q$, iff every interpretation that satisfies $\cK$ also satisfies $Q$.

	We consider Boolean queries $Q$ over a KB $\cK$, i.e., $Q$ is an axiom for which we want to test the entailment from the KB, written as $\cK \models Q$. The entailment test may be reduced to checking the unsatisfiability of a 
	KB or the unsatisfiability of a concept w.r.t. a KB. A concept is satisfiable in a KB \cK if there exists an interpretation \cI such that $C^{\cI}\neq \emptyset$ and $\cI\models \cK$. For example, the entailment of the axiom $C\sqsubseteq D$ by a KB \cK may be tested by checking the unsatisfiability of the concept $C \sqcap \neg D$ w.r.t. \cK, 
	while the entailment of the axiom $a:C$ may be tested by checking the unsatisfiability of the KB with the addition of $a: \neg C$.

	\subsection{Probabilistic Description Logics}
	\label{sec:prob-descr-logics}

	DISPONTE~\cite{RigBelLamZes15-SW-IJ,Zese17-SSW-BK} is based on the distribution semantics \cite{DBLP:conf/iclp/Sato95} and  allows the user to label some axioms with real values $p \in [0,1]$ representing probabilities.
	In DISPONTE, a probabilistic  knowledge base \cK is a set of probabilistic axioms and a set of certain axioms.
	A \textit{certain axiom} takes the form of a regular DL axiom.
	A \textit{probabilistic axiom} takes the form $p::E$, where $p$ is a probability and $E$ a regular axiom. The probability $p$ can be interpreted as an \emph{epistemic probability}, i.e., as the degree of our belief in evidence for axiom $E$. For example, a probabilistic assertion axiom $p :: a : C$ means that we have evidence for $C(a)$ with degree of belief $p$. This $p$ is not the probability of the truth of axiom $E$ but the probability of evidence of axiom $E$.  In case we have two independent sources of evidence for the same axiom $E$, one with probability $p_1$ and one with probability $p_2$, we would add the probabilistic axiom	$1-(1-p_1)(1-p_2)::E$ to \cK because the different bits of evidence are independent.
	
	The idea of DISPONTE is to associate independent Boolean random variables to the probabilistic axioms. By assigning values to every random variable, we obtain a \emph{world}, i.e., a non-probabilistic KB  containing all the certain axioms and the set of axioms whose random variables are assigned the value 1.
	Given a query $Q$, DISPONTE defines the probability of $Q$ by constructing worlds. 
	If $p_i :: E_i$ is the $i$th probabilistic axiom in the KB, every world $w$ has a probability $P(w) = \prod_{E_i\in w}p_i\times\prod_{E_i\not\in w}(1-p_i)$, because the presence of the axioms is  considered pair-wise independent.
	$P(w)$ is a probability distribution over worlds.
	
	Given a  consistent KB $\cK$, a query $Q$, and the set of all worlds  $\cW_\cK$, the probability of $Q$ is the sum of the probabilities of the worlds in which the query is true~\cite{RigBelLamZes15-SW-IJ}:
	$$
	P(Q) = \sum_{w\in \cW_\cK : w\models Q}P(w)
	$$
	\noindent
	So, $P(Q)$ is the probability of $Q$ w.r.t. the entire consistent KB, i.e., every axiom of the KB can contribute to the query probability.

	\begin{exa}[Flying Penguins - 1]
		\label{ex:pingu-prob}
	
		Consider the following KB:
		\setcounter{equation}{1}
		\begin{align}
			&(1)\ 0.9:: \mathit{Penguin} \sqsubseteq\mathit{Bird}\notag\\
			&(2)\ 0.9:: \mathit{Penguin} \sqsubseteq\neg\mathit{Fly}\notag\\
			&(3)\ 0.6:: \mathit{pingu}:	\mathit{Penguin}\notag
		\end{align}
		
		The first two axioms state that penguins are birds and that penguins do not fly, both with probability $0.9$. The third states that $\mathit{pingu}$, an individual, is a penguin with probability $0.6$.
		This KB is consistent and has 8 possible worlds because there are 3 probabilistic axioms, and each of them may be contained in a world or not. Thus, there are $2^3$ possible worlds, which are all the possible combinations given by the probabilistic axioms. Let us consider the query $Q=\mathit{pingu}:\mathit{Bird}$. This query is true in two worlds, those containing the first and third axioms. The probability is $P(Q) = 0.9\cdot 0.9\cdot 0.6 + 0.9\cdot0.1\cdot 0.6 = 0.54$.
	\end{exa}

	\begin{exa}[Flying Penguins - 1.1]
	
		Consider the slightly different KB:
		\setcounter{equation}{1}
		\begin{align}
			&(1)\ 0.9:: \mathit{Penguin} \sqsubseteq\mathit{Bird}\notag\\
			&(3)\ 0.6:: \mathit{pingu}:	\mathit{Penguin}\notag\\
			&(4)\ 0.6:: \mathit{pingu}:	\mathit{Bird}\notag
		\end{align}
		This KB differs from that of Example \ref{ex:pingu-prob} because we replaced axiom (2) with axiom (4) expressing that $\mathit{pingu}$ is a bird with probability $0.6$.
		Let us consider the query $Q=\mathit{pingu}:\mathit{Bird}$. We have 8 possible worlds, and $Q$ is true in 5 of them, i.e., all the worlds where there are both axioms (1) and (3), as in Example \ref{ex:pingu-prob}, or there is axiom~(4). In this case, $P(Q)= 0.9\cdot 0.6\cdot 0.6 + 0.9\cdot0.6\cdot 0.4 + 0.1\cdot0.4\cdot 0.6+ 0.1\cdot0.6\cdot 0.6+ 0.9\cdot0.4\cdot 0.6= 0.816$, which is higher than the probability of axiom (4). This is due to the fact that each axiom is independent from the other and each axiom contributes to the query probability.
	\end{exa}

	However, computing the probability of queries building all the worlds is infeasible because their number is exponential in the number of probabilistic axioms in the KB.  To try to circumvent this problem, it is possible to resort to classical inference algorithms for collecting justifications for the query, which are usually less than the worlds in the case of large KBs, and compute the probability from them. The problem of finding justifications for a query has been investigated by various authors~\cite{DBLP:conf/ijcai/SchlobachC03,extended_tracing,DBLP:journals/jar/HorrocksS07,DBLP:conf/cade/SebastianiV09,DBLP:journals/jar/BaaderP10,DBLP:journals/logcom/BaaderP10,DBLP:conf/sat/ArifMM15}. 
	
	A \emph{justification} for a query $Q$ is a consistent inclusion-minimal subset $\cE$ of logical axioms  of a KB $\cK$ such that $\cE \models Q$.  $\textsc{all-just}(Q,\cK)$ indicates the set of all justifications for the query $Q$ in the KB $\cK$. On the other hand, a justification for the inconsistency of a KB is an inclusion-minimal subset $\cE$ of logical axioms of a KB $\cK$ such that $\cE$ is inconsistent. 
	$\textsc{all-just}(\mathit{Incons},\cK)$ indicates the set of all the justifications for the inconsistency of the KB $\cK$.  
	Note that, if the KB is consistent, there is no justification for the inconsistency. On the other hand, if the KB is inconsistent there will be at least one justification $\cE$ for the inconsistency. 
	Also note that, if the KB is inconsistent, every query $Q$ is formally entailed, even if there are no justifications for $Q$.
	
	A set $\cJ$ of justifications defines a set of worlds $\cW_\cJ=\{w|j\in\cJ,j\subseteq w\}$. $\cJ$ is called \emph{covering} for $Q$ if $\cW_\cJ$ is equal to the set of all the consistent worlds in which $Q$ succeeds, i.e., if for each consistent $w$, $w\models Q$ $\leftrightarrow w\in W_\cJ$; or \emph{covering} for the inconsistency of $\cK$ if it identifies all the inconsistent worlds.

	\begin{exa}[Flying Penguins - 2]
		\label{ex:pingu-expl}
		Let us consider the KB of Example~\ref{ex:pingu-prob} and the query $Q=\mathit{pingu}:\mathit{Bird}$. This query has one justification, i.e., $\{(1),(3)\}$, which is also covering because it defines the two worlds where the query holds.
	\end{exa}

	An effective way of computing the probability of a query from a covering set of justifications consists in compiling it into a Binary Decision Diagram (BDD).
	A BDD is a rooted graph used to represent a function of Boolean variables, with one level for each variable. Each node of the graph has two children corresponding either to the 1 value or the 0 value of the variable associated with the node. Its leaves are either 0 or 1.
	BDDs are used to represent the Disjunctive Normal Form (DNF) Boolean formula $f_{\textsc{all-just}(Q,\cK)}(\mathbf{X})$ built from $\textsc{all-just}(Q,\cK)$.
	\begin{defi}[Boolean formula of a justification]\label{def:bool-f-j}
		Given a justification $\cE$, and a set $\mathbf{X}=\{ X_i\mid p_i::E_i \in \cK\}$ of independent Boolean random variables associated with probabilistic axioms with $P(X_i=1)=p_i$, where $p_i$ is the probability of axiom $E_i$, the Boolean formula of the justification is $f_{\cE}(\mathbf{X})=\bigwedge_{(E_i\in\cE)}X_{i}$.
	\end{defi}
	\begin{defi}[Boolean formula of a set of justifications]\label{def:bool-f-set-j}
		Given a set of justifications $\cJ$, and a set $\mathbf{X}=\{ X_i\mid p_i::E_i \in \cK\}$ of independent Boolean random variables associated with probabilistic axioms with $P(X_i=1)=p_i$, where $p_i$ is the probability of axiom $E_i$, the Boolean formula of the set of justifications is $f_{\cJ}(\mathbf{X})=\bigvee_{\cE\in\cJ}f_{\cE}(\mathbf{X})=\bigvee_{\cE\in\cJ}\bigwedge_{(E_i\in\cE)}X_{i}$.
	\end{defi}
	From Definition~\ref{def:bool-f-set-j}, given a query $Q$ and the set $\textsc{all-just}(Q,\cK)$ of all the justifications for $Q$ in the consistent KB $\cK$, the formula $f_{\textsc{all-just}(Q,\cK)}(\mathbf{X})$ is $\bigvee_{\phi\in\textsc{all-just}(Q,\cK)}\bigwedge_{(E_i\in\phi)}X_{i}$.
	The probability that $f_{\textsc{all-just}(Q,\cK)}(\mathbf{X})$ takes value 1 gives the probability of $Q$~\cite{RigBelLamZes15-SW-IJ,Zese17-SSW-BK}.  The same applies also for inconsistency, 
	where $f_{\textsc{all-just}(\mathit{Incons},\cK)}(\mathbf{X})$ is the DNF Boolean formula representing $\textsc{all-just}(\mathit{Incons},\cK)$. 
	\begin{defi}[Probability of a KB being inconsistent]\label{def:prob-kb-incons}
		The probability of a KB \cK being inconsistent is defined as the probability that $f_{\textsc{all-just}(\mathit{Incons},\cK)}(\mathbf{X})$ takes value 1.
	\end{defi}
	This can be seen as a measure of the inconsistency of the KB.
	
	In principle, to compute the probability we could resort to different possible languages, such as Sentential Decision Diagrams (SDD) or Deterministic Decomposable Negation Normal Form (d-DNNF). We use BDDs because 
	packages for compiling formulas into BDDs are extremely optimized and can manage BDDs of very large size (see e.g.~\cite{DBLP:conf/ijcai/RaedtKT07,DBLP:journals/tplp/KimmigDRCR11,riguzzi2011pita}) while performing sometimes better than SDD packages~\cite{DBLP:journals/tplp/KieselTK22}.
	
	Given the BDD, we can use function \textsc{Probability} described by Kimmig et al. \cite{DBLP:journals/tplp/KimmigDRCR11} to compute the probability.

	\subsection{Repair Semantics}
	In this section we briefly recall the main  definitions and semantics of repairs assuming the TBox consistent.
	\begin{defi}[Repair]
		A repair $\cR$ of a KB $\cK$ is an inclusion-maximal subset of the ABox that is consistent together with the TBox.
	\end{defi}
	Given a query $Q$ and a KB $\cK=(\cA,\cT)$, where $\cA$ is the ABox and $\cT$ is the TBox, we can define the three main semantics for repairs as follows. 
	\begin{defi}[Repair Semantics]\hfill
		\begin{description}
			\item[Brave] A query $Q$ is true over a KB $\cK$ under the Brave semantics, written $\cK\models_{Brave}Q$, if $(\cR,\cT)\models Q$ for at least one repair $\cR$ of $\cK$~\cite{DBLP:conf/ijcai/BienvenuR13}.	
			\item[AR] A query $Q$ is true over a KB $\cK$ under the AR  semantics, written $\cK\models_{AR}Q$, if $(\cR,\cT)\models Q$ for every repair $\cR$ of $\cK$~\cite{DBLP:conf/rr/LemboLRRS10}.
			\item[IAR] A query $Q$ is true over a KB $\cK$ under the IAR  semantics, written $\cK\models_{\mathit{IAR}}Q$, if $(\cD,\cT)\models Q$ where $\cD=\bigcap_{\cR\in Rep(\cK)}\cR$ and $Rep(\cK)$ is the set of all the repairs for $\cK$~\cite{DBLP:conf/rr/LemboLRRS10}.
		\end{description}
		
	\end{defi}
	
	\begin{exa}[University employee positions]\label{ex:univ-empl}
		Consider the following KB $\cK=(\cA,\cT)$, where $\cT$ is:
		\begin{align}
			&(1)\ \mathit{Professor} \sqcap \mathit{Tutor} \sqsubseteq \mathit{Lecturer}\notag\\
			&(2)\ \mathit{Person} \sqcap \mathit{Professor} \sqsubseteq \mathit{PhD}\notag\\
			&(3)\ \mathit{Professor} \sqcup \mathit{Tutor} \sqsubseteq \mathit{UniversityEmployee}\notag\\
			&(4)\ \mathit{Professor} \sqsubseteq \neg\mathit{Tutor}\notag
		\end{align}
		\noindent
		and $\cA$ is:
		\begin{align}
			&(5)\ \mathit{alice} : \mathit{Person}\notag\\
			&(6)\ \mathit{alice} : \mathit{Professor}\notag\\
			&(7)\ \mathit{alice} : \mathit{Tutor}\notag
		\end{align}
		The KB states that a professor which is also a tutor is a lecturer (1), that a person who is a professor is a PhD (2), and that professors and tutors are university employees (3). A professor cannot be a tutor (4). Finally, Alice is a person (5), a professor (6) and a tutor (7).
		
		It is easy to see that the ABox of this KB is inconsistent w.r.t. the TBox.
		This KB has two repairs: (I) $\cR_I=\cA\setminus\{(6)\}$ and (II) $\cR_{II}=\cA\setminus\{(7)\}$.
		The query $Q_1=\mathit{alice}:\mathit{Lecturer}$ is false under the three semantics because it is impossible to find a repair where Alice is both a professor and a tutor.
		The query $Q_2=\mathit{alice}:\mathit{PhD}$ is true under the Brave semantics because it is true in $\cR_{II}$.
		The query $Q_3=\mathit{alice}:\mathit{UniversityEmployee}$ is true under the AR semantics because it is true in every repair.
		Finally, the query $Q_4=\mathit{alice}:\mathit{Person}$ is true under the IAR semantics because it is true in the intersection of $\cR_I$ and $\cR_{II}$, that contains only axiom~$(5)$. 
	\end{exa}

	\section{Querying Inconsistent Knowledge Bases}
	\label{sec:query-inc-kb}
	
	We can use the definitions from Section~\ref{sec:prob-descr-logics} to define the probability of a query $Q$ given a (probabilistic) possibly inconsistent KB $\cK$, with $Q$ an axiom. In Section~\ref{sec:comp-repair}, we will compare our proposal with the repair semantics~\cite{DBLP:conf/rweb/BienvenuB16}.
	We use a probability value  strictly smaller than $1.0$
	to mark which axioms may be incorrect  and so can be removed to debug (or repair) the KB. 
	Under this semantics, given a KB $\cK$ and a query $Q$, the probability of $Q$ is:
	$$P_C(Q)=P(Q|\mathit{Cons})=\frac{P(Q,\mathit{Cons})}{P(\mathit{Cons})}$$
	Here, $P(\mathit{Cons})$ is the probability that the KB is consistent, i.e., the probability of the formula $\neg f_{\textsc{all-just}(\mathit{Incons},\cK)}$, while $P(Q,\mathit{Cons})$ is the probability of the formula $f_{\textsc{all-just}(Q,\cK)}\wedge\neg f_{\textsc{all-just}(\mathit{Incons},\cK)}$. 
	This represents all the \emph{consistent worlds} and checks if the query holds in each.
	The final probability is the probability of the query within the consistent worlds over the probability of the consistency of the KB.
	
	Moreover, from the definition of $P_C(Q)$, if the KB is certainly inconsistent ($P(\mathit{Cons})=0.0$), $P_C(Q)$ is not defined. This reflects the fact that, if a KB is certainly inconsistent, then every query is entailed and so the KB does not provide meaningful information. Such a KB can be seen as a KB that has no repair.

	\begin{exa}[Flying Penguins - 3]
		\label{ex:pingu-incons}
		Consider the query $Q=\mathit{pingu}:\neg\mathit{Fly}$ and the following KB (slightly different from that of Ex.~\ref{ex:pingu-prob}): 
		\setcounter{equation}{0}
		\begin{align}
			(1)\>0.9::\ &\mathit{Bird} \sqsubseteq\mathit{Fly}&&(2)\>\mathit{Penguin} \sqsubseteq\mathit{Bird}\notag\\
			(3)\>0.9::\ &\mathit{Penguin} \sqsubseteq\neg\mathit{Fly}&&(4)\>\mathit{pingu}:	\mathit{Penguin}\notag
		\end{align}
		In this case we have four different worlds:
		\begin{align*}
			w_1 =&\{(1),(2),(3),(4)\}&w_2 =&\{(1),(2),(4)\}\\
			w_3 =&\{(2),(3),(4)\}&w_4 =&\{(2),(4)\}
		\end{align*}
		World $w_1$ is inconsistent, therefore it does not contribute to the probability of $(Q,\mathit{Cons})$. Among the other worlds, the query $Q$ is true only in $w_3$, which has probability $0.9\cdot0.1=0.09$. The probability of the KB to be consistent is
		\begin{align}
			P(w_2)+P(w_3)+P(w_4)=&\ (P(1)\cdot(1-P(3)))\ +\notag\\
			&\ ((1-P(1))\cdot P(3))\ +\notag\\
			&\ ((1-P(1))\cdot(1-P(3)))\notag\\
			=&\ 0.9\cdot0.1+0.1\cdot0.9+0.1\cdot0.1\notag\\
			=&\ 0.19\notag
		\end{align}
		So, $P_C(Q)=0.09/0.19=0.474$.
		
		An interesting result can be seen if axiom (3) is non-probabilistic. In this case, the KB has the two worlds $\{(1),(2),(3),(4)\}$ and $\{(2),(3),(4)\}$, having probability $0.9$ and $0.1$, respectively. The first world is inconsistent while $Q$ holds in the second. The probability of  consistency is $P(\mathit{Cons})=0.1$ because only the second world is consistent. In this world the query is true, so the probability $P(Q,\mathit{Cons})=0.1$. As result, we obtain that $P_C(Q)=1$. On the other hand, the query $\overline{Q}=\mathit{pingu}:\mathit{Fly}$ takes probability $0$. The same results can be achieved with every value of probability of axiom (1) that is strictly lower than 1. This example shows how a correct design of the knowledge is important. Indeed, by associating a probability value to axiom (1), which is not always true in the domain, axioms that are certain acquire, in a way, more importance. The information that penguins do not fly is certain because there are no species of penguins that have the ability to fly, thus, irrespectively of the probability of axiom (1), the query $Q$ is certainly true.
	\end{exa}

	This process keeps the probability of queries w.r.t. a consistent KB unchanged. The proof of this statement is trivial because, if the KB is consistent, all the worlds are so, and therefore the computation of the query is equivalent to the former definition of the DISPONTE semantics.

	\subsection{Reasoning Algorithm}
	\label{sec:reas-alg}

	One of the most used approaches to compute justifications is the tableau algorithm~\cite{DBLP:conf/ijcai/SchlobachC03,extended_tracing,DBLP:journals/jar/HorrocksS07}. Sebastiani and Vescovi~\cite{DBLP:conf/cade/SebastianiV09} defined an approach for finding justifications in the $\mathcal{EL}$ DL that builds a Horn propositional formula of polynomial size and applies Boolean Constraint Propagation. Arif et al.~\cite{DBLP:conf/sat/ArifMM15} used implicit hitting set dualization by exploiting a SAT solver. Baader and colleagues~\cite{DBLP:journals/jar/BaaderP10,DBLP:journals/logcom/BaaderP10} presented different approaches that create a Boolean formula, called \emph{pinpointing formula}, which represents the set of all the justifications for a query w.r.t. $\mathcal{SI}$ KBs. This approach is implemented, for example, in TRILL$^P$ and TORNADO~\cite{DBLP:journals/tplp/ZeseCLBR19}, two sub-systems of TRILL. These two sub-systems are not considered in this paper because the first needs a further processing of the results to apply our approach, while the second returns results in a form not suitable for our purposes.

	We now describe the tableau approach that is used in the reasoners we extended, TRILL~\cite{DBLP:journals/tplp/ZeseCLBR19,ZesBelRig16-AMAI-IJ,DBLP:journals/ws/ZeseC21} and BUNDLE~\cite{RigBelLamZese13-RR13b-IC,CotRigZesBelLam-IC-2018}, then we describe how we can collect the justifications for both the queries and the inconsistency in order to correctly compute the probability and correctly answer the query itself.
	
	A \textit{tableau} is a graph where each node represents an individual $a$ and is labelled
	with the set of concepts $\cL(a)$ to which $a$ belongs. Each edge
	$\langle a,b\rangle$ in the graph is labelled with the set of roles $\mathcal{L}(\langle a,b\rangle)$ to
	which the couple $(a, b)$ belongs. 
	The algorithm proves an axiom by contradiction by repeatedly applying a set of consistency preserving \emph{tableau expansion rules} until a clash (i.e., a contradiction) is detected or a clash-free graph is found to which no more rules are applicable.
	A clash is a couple $(C, a)$ where $C$ and $\neg C$ are
	present in the label of the node $a$, i.e., $\{C, \neg C\} \subseteq \cL(a)$.
	
	The expansion rules modify the labels of the nodes and edges of the tableau and update a \emph{tracing function} $\tau$, which associates a set of justifications to each concept or role of a label in the tableau. For example, the value of the tracing function for the label $C$ of node $a$ when axiom $a:C$ is in the KB is set to $\{\{a:C\}\}$, i.e., a set of justifications containing one single justification that, in turns, contains the axiom $a:C$.
	Given a query, once the tableau is fully expanded, to build the justification for the query, the labels that cause clashes are collected. Then, the justifications of these labels are joined to form the justification for the query. We refer the interested reader to~\cite{Zese17-SSW-BK} for a detailed overview.

	The expansion rules 
	are divided into \emph{deterministic} and \emph{non-deterministic}. As stated above, the first, when applied to a tableau, produce a single new tableau. The latter, when applied to a tableau, produce a set of tableaux.

	When the tableau algorithm adds the negation of the query to the tableau,  the value of its tracing function is set to $\{\emptyset\}$, i.e., an empty justification, because that information does not come from the KB. So, if a clash is detected, it is not possible to know what causes it: an inconsistency, or the query. Thus, reasoners usually perform a consistency check before expanding the tableau, preventing its execution if the KB is inconsistent.
	This is due to the fact that, when the axioms involved in the justification of a query also cause an inconsistency,  the justifications for the query may be subsets of those for the inconsistency. In Example~\ref{ex:pingu-incons}, given the query $Q=\mathit{pingu}:\mathit{Fly}$, the single justification for the query is the set of axioms $\{(4),(2),(1)\}$, while that for the inconsistency is $\{(4),(2),(1),(3)\}$. In this case, when extracting the justifications from the values of the tracing function of the labels that create the clashes, the latter is not collected because it is a superset of the first.
	
	In order for the reasoner to collect justifications for both the query and the inconsistency, a simple approach is to change the tracing function so that it adds a placeholder to the justifications for the negation of the query, in order to separate the justifications that contain this placeholder from those without it, which are justifications caused by an inconsistency.
	Basically, the tracing function for the negation of the query is initialized as $\{\{Q_p\}\}$, where $Q_p$ is a fake axiom that does not appear in the KB. This axiom represents a flag indicating that the label has been created because of the query.
	Then, the standard expansion rules can be applied to expand the tableau in the usual way, because $Q_p$ acts like an axiom. Therefore, the tableau algorithm remains unchanged. At the end, if there are clashes, the justifications for the query will contain $Q_p$, those due to the inconsistency will not. 
	An important aspect of this implementation is that all the results about completeness and correctness of the tableau (see \cite{horrocks2006even}) still apply, since the tableau algorithm is not modified. This also possibly allows the application of this approach to every DL for which tableau expansion rules have been defined.
	
	\begin{exa}[Flying Penguins - 4]
		Consider the KB of Example~\ref{ex:pingu-incons} where the probability values from the axioms have been removed.
		\begin{align}
			(1)\>&\mathit{Bird} \sqsubseteq\mathit{Fly}&&(2)\>\mathit{Penguin} \sqsubseteq\mathit{Bird}\notag\\
			(3)\>&\mathit{Penguin} \sqsubseteq\neg\mathit{Fly}&&(4)\>\mathit{pingu}:	\mathit{Penguin}\notag
		\end{align}
		\noindent
		By using standard tableau algorithms, we cannot ask any query because the KB is inconsistent. So, let us remove axiom $(3)$ for the moment. The classic tableau algorithm creates a tableau containing a single node, corresponding to $\mathit{pingu}$, labelled with the concept $\mathit{Penguin}$, having as value of $\tau$ the set of justifications $\{\{(4)\}\}$.
		Given the query $Q=\mathit{pingu}:\mathit{Fly}$, the tableau is updated by adding the label $\neg\mathit{Fly}$ to the node of  $\mathit{pingu}$, with the value of $\tau$ equals to $\{\emptyset\}$.
		Expanding this tableau by following the axioms of the KB, the label $\mathit{Fly}$ will be added to the node for $\mathit{pingu}$, with the value for $\tau$ corresponding to $\{\{(4),(2),(1)\}\}$. In this tableau there is a clash because the node for $\mathit{pingu}$ contains both the labels  $\mathit{Fly}$ and $\neg\mathit{Fly}$. A justification can be found by joining the values of the tracing function of the two labels, i.e., $\emptyset\cup\{(4),(2),(1)\}=\{(4),(2),(1)\}$.
		
		Let us now reintroduce axiom $(3)$. During the expansion of the tableau, the value of $\tau$ for the label $\neg\mathit{Fly}$ is updated by adding the justification $\{(4),(3)\}$. However, this justification cannot be added because the initial value of the tracing function contains the empty set, which is a subset of any other set, so the tableau cannot correctly discriminate between justifications due to the query and to the inconsistency.  This explains why classic tableau algorithms require consistent KBs.

		If we consider this example with the new tracing function, the value of $\tau$ for the label of the query $\neg\mathit{Fly}$ is initialized as $\{\{Q_p\}\}$. Therefore, the justifications for $Q$ will be $\{\{Q_p,(4),(2),(1)\}\}$, while those for the inconsistency will be $\{\{(4),(2),(1),(3)\}\}$. In this case, during the expansion of the tableau, the value of $\tau$ for the label $\neg\mathit{Fly}$ can be updated by adding the justification $\{(4),(3)\}$, because the justification due to the inconsistency is no more a superset of $\{Q_p\}$, which is due to the query, and the reasoner can easily discriminate between the two.  Once all the justifications are collected, the reasoner implementing this tracing function can correctly answer query $Q$, returning in this case \emph{undefined} because the inconsistency derives from the axioms also required to prove the query and all the axioms are certain. Thus, differently from Example~\ref{ex:pingu-incons}, it is not possible to define whether $Q$ is true or not, nor to compute $P_C(Q)$ because $P(\mathit{Cons})=0.0$.
	\end{exa}
	
	Another possible way for the tableau algorithm to discriminate and collect justifications for queries and inconsistency  is to directly add the negation of the query to the KB by using a name known by the reasoner. For example, if the query is the axiom $a:C$, it is sufficient to add the axioms $a:C_{Q_p}$ and $C_{Q_p} \sqsubseteq \neg C$ where $C_{Q_p}$ is a fresh concept not contained in the KB. Now, it is possible to run a reasoner that can return the justifications for the inconsistency of a KB and split all the justifications in two sets, one containing justifications with axioms involving the concept $C_{Q_p}$ and one containing those with axioms not involving the fresh concept. The first set will contain the justifications for the query (in our case the axioms $a:C_{Q_p}$ and $C_{Q_p} \sqsubseteq \neg C$ must be removed from the justifications in order to collect justifications containing only axioms from the original KB), while the second will contain the justifications for the inconsistency. In this way, we do not need to modify the reasoner, nor the tracing function, making the extension to the tableau algorithm even easier to implement.

	\begin{figure}[t]
		\centering
		\includegraphics[width=\linewidth]{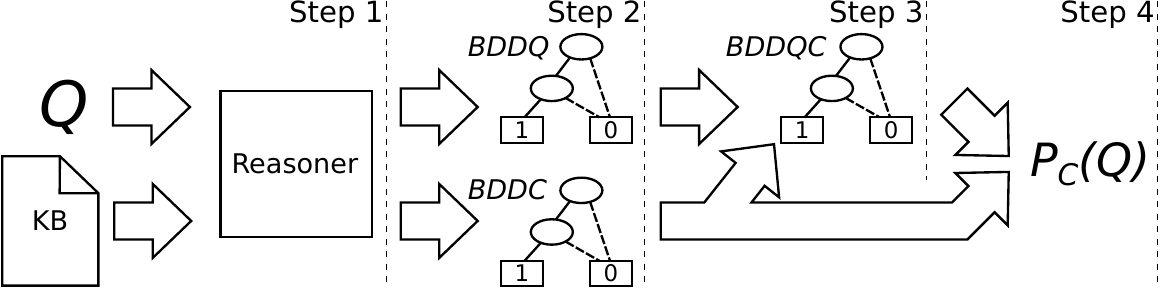}
		\caption{Reasoning flow: Step 1 computes the set of justifications for the query $Q$ w.r.t. the KB $\cK$ $\textsc{all-just}(Q,\cK)$ and for the inconsistency (if any) $\textsc{all-just}(\mathit{Incons},\cK)$. Step 2 builds the BDDs $BDDQ$ from $\textsc{all-just}(Q,\cK)$ and $BDDC$ from $\textsc{all-just}(\mathit{Incons},\cK)$. Step 3 joins $BDDQ$ and $BDDC$ creating $BDDQC$. Finally, Step 4 computes the probability $P_C(Q)$ by first computing the probabilities $P(Q,\mathit{Cons})$ from $BDDQC$ and $P(\mathit{Cons})$ from $BDDC$.\label{fig:reason-flow}}
	\end{figure}
	
	The whole reasoning flow can be divided into 4 different steps, described below, and shown in Figure~\ref{fig:reason-flow}.
	Once all the justifications are collected, the BDD $BDDQ$ for the query $Q$ is built from $\textsc{all-just}(Q,\cK)$ while that for the consistency of the KB $BDDC$ is generated by negating the BDD for $\textsc{all-just}(\mathit{Incons},\cK)$. The next step consists of joining $BDDQ$ and $BDDC$ to create a BDD $BDDQC=BDDQ\wedge BDDC$ from which the probability $P(Q,\mathit{Cons})$ can be computed. The probability $P(\mathit{Cons})$ can be computed directly from $BDDC$. Finally, the probability $P_C(Q)$ is computed from the two BDDs as $P(Q,\mathit{Cons})/P(\mathit{Cons})$.

	\begin{exa}[Probabilistic university employee positions]\label{ex:univ-empl-prob}
		Consider a KB $\cK=(\cA,\cT)$ that differs from the one of Example~\ref{ex:univ-empl} because all the ABox axioms are  probabilistic, so $\cT$ is:
		\begin{align}
			&(1)\ \mathit{Professor} \sqcap \mathit{Tutor} \sqsubseteq \mathit{Lecturer}\notag\\
			&(2)\ \mathit{Person} \sqcap \mathit{Professor} \sqsubseteq \mathit{PhD}\notag\\
			&(3)\ \mathit{Professor} \sqcup \mathit{Tutor} \sqsubseteq \mathit{UniversityEmployee}\notag\\
			&(4)\ \mathit{Professor} \sqsubseteq \neg\mathit{Tutor}\notag
		\end{align}
		\noindent
		and $\cA$ is:
		\begin{align}
			(5)\ 0.9::&\ \mathit{alice} : \mathit{Person}\notag\\
			(6)\ 0.2::&\ \mathit{alice} : \mathit{Professor}\notag\\
			(7)\ 0.8::&\ \mathit{alice} : \mathit{Tutor}\notag
		\end{align}
		The only justification for the inconsistency is $\{(6),(7),(4)\}$, so $P(\mathit{Incons})=0.2\cdot0.8=0.16$.
		Consider also a second KB $\cK'=(\cA,\cT\setminus\{(4)\})$: $\cK'$ is consistent, so every query is true under IAR semantics.
		
		Consider query $Q_1=\mathit{alice}:\mathit{Lecturer}$: it is false  w.r.t. $\cK$ under all three semantics because it is impossible to find a repair where Alice is both a professor and a tutor. The set of justifications for $Q_1$ is $\{\{(6),(7),(1)\}\}$. There is no consistent world where $Q_1$ is true  w.r.t. $\cK$, so $P_C(Q_1)=0.0$.  If we consider $\cK'$, then $\textsc{all-just}(Q_1,\cK')=\textsc{all-just}(Q_1,\cK)$ while $P_C(Q_1)=P(Q_1)=0.16$.
		The query $Q_2=\mathit{alice}:\mathit{PhD}$ is true under the Brave semantics because it is true in $\cR_{II}=\cA\setminus\{(7)\}$. Its set of justifications is $\{\{(5),(6),(2)\}\}$ and $P_C(Q_2)=0.04286$. If we consider $\cK'$, we have $\textsc{all-just}(Q_2,\cK')=\textsc{all-just}(Q_2,\cK)$, while $P_C(Q_2)=P(Q_2)=0.18$.
		Consider query $Q_3=\mathit{alice}:\mathit{UniversityEmployee}$: it is true under the AR semantics because it is true in every repair but not in their intersection. Its set of justifications is $\{\{(6),(3)\},\{(7),(3)\}\}$ and $P_C(Q_3)=0.80952$. If we consider $\cK'$, we have $\textsc{all-just}(Q_3,\cK')=\textsc{all-just}(Q_3,\cK)$, while $P_C(Q_3)=P(Q_3)=0.84$.
		Finally, the query $Q_4=\mathit{alice}:\mathit{Person}$ is true under the IAR semantics because it is true in the intersection of $\cR_I=\cA\setminus\{(6)\}$ and $\cR_{II}$, that contains only axiom~$(5)$. Its set of justifications is $\{\{(5)\}\}$ and $P_C(Q_4)=0.9$. If we consider $\cK'$, then $\textsc{all-just}(Q_4,\cK')=\textsc{all-just}(Q_4,\cK)$, while $P_C(Q_4)=P(Q_4)=0.9$ as for $\cK$.
		
		This example shows that the probability of a query does not change if the axioms causing the inconsistency are not used to explain the query. In all the other cases, the probability decreases depending on the probability of the inconsistency. Note that, if $P(\mathit{Incons})=1.0$, then the probability of any query cannot be computed because the KB is certainly inconsistent. However, since we use the probability to tell which axioms can be removed to create the repairs, $P(\mathit{Incons})=1.0$ means that the axioms causing the inconsistency are not probabilistic, so they cannot be removed from the KB to create the repairs. Therefore, in this case the query will be false under the repair semantics.
	\end{exa}

	Regarding complexity, there are two problems to consider. The first is finding justifications, whose complexity has been deeply studied and depends on the logic used~\cite{DBLP:conf/dlog/PenalozaS09,DBLP:conf/dlog/PenalozaS10,DBLP:conf/ecai/PenalozaS10}. In particular, Corollary 15 in \cite{DBLP:conf/ecai/PenalozaS10} shows that 
	finding all the justifications cannot be solved in output polynomial time for \textit{DL-Lite$_{bool}$} TBoxes unless $P = NP$.
	Since the proof of this corollary for \textit{DL-Lite$_{bool}$} uses only constructs available in \alc, this result also holds for \alc and all its extensions, making the problem of finding justifications not efficiently solvable.
	Despite these results, it has been shown that all justifications can be found over many real-world ontologies within a  few seconds \cite{Kalyanpurphd,DBLP:conf/semweb/KalyanpurPHS07}.
	
	The second problem is building the BDD and computing the probability, which can be seen as the problem of computing the probability of a sum-of-products \cite{RCDB03}. While the problem is \#P-hard\footnote{The class \#P \cite{DBLP:journals/siamcomp/Valiant79} describes counting problems associated with decision problems in NP. More formally, \#P is the class of function problems  counting the number of accepting paths of a nondeterministic Turing machine running in polynomial time.
		A prototypical \#P problem is the one of computing the number of satisfying assignments of a CNF Boolean formula.
		\#P problems have the following characteristics:
		first, a \#P problem must be at least as hard as the corresponding NP problem; second, Toda \cite{DBLP:conf/focs/Toda89} showed that a polynomial-time machine with a \#P oracle ($\mathrm{P^{\#P}}$) can solve all problems in the polynomial hierarchy PH.}, algorithms based on BDDs were able to solve problems with hundreds of thousands of variables (see e.g. the works on inference on probabilistic logic programs
	\cite{DBLP:conf/ijcai/RaedtKT07,Rig-AIIA07-IC,Rig09-LJIGPL-IJ,RigSwi10-ICLP10-IC,DBLP:journals/tplp/KimmigDRCR11,riguzzi2011pita}). Methods for weighted model counting \cite{DBLP:conf/aaai/SangBK05,DBLP:journals/ai/ChaviraD08} can be used as well to solve the \textsc{sum-of-products} problem.
	
	Note that the extensions introduced in this paper do not change the complexity of the two problems, and the original algorithms (those without the extensions) proposed for solving these two problems were shown to be able to work on inputs of real world size\footnote{E.g., NCI ontology (\url{https://ncit.nci.nih.gov/ncitbrowser/}) with 3,382,017 axioms and $\mathcal{SH}$ expressiveness, or FMA (\url{http://si.washington.edu/projects/fma}) with 88,252 axioms in the TBox and RBox and 237,382 individuals and $\mathcal{ALCOIN}(\mathbf{D})$ expressiveness.}~\cite{Zese17-SSW-BK} consisting in consistent KBs with high number of axioms.

	\section{Comparison with the Repair Semantics}
	\label{sec:comp-repair}
	In this section we compare our proposal with the repair semantics~\cite{DBLP:conf/rweb/BienvenuB16}, where the authors consider KBs where the TBox is consistent while the ABox may be inconsistent w.r.t. the TBox. 
	
	We construct the worlds by including all the TBox axioms and some axioms from the ABox, so a world differs from a repair. 
	
	As already discussed, a world is a subset of axioms of the original KB, i.e., it can be seen as a smaller KB or a sort of repair. A \emph{consistent world} is a world that is consistent. Conversely, an \emph{inconsistent world} is a world that is inconsistent.
	Since repairs correspond to all  inclusion-maximal consistent subsets of ABox axioms, they can be viewed as `possible worlds'~\cite{DBLP:journals/jair/BienvenuBG19}. A repair is an inclusion-maximal subset of the ABox that is TBox-consistent\footnote{An ABox \cA is TBox-consistent if the KB $\cK=(\cA,\cT)$ is consistent, otherwise the ABox is TBox-inconsistent.}, 
	so, given a repair, we can build a consistent KB having the entire TBox and the set of axioms from the ABox selected by the repair.
	Therefore, given $Rep(\cK)$, the set of all repairs, \cT the TBox of the KB \cK, and $\cW_\cK$ the set of all worlds, we have $\{w_{\cR_i} = \cR_i\cup\cT, \forall \cR_i\in Rep(\cK)\}\subseteq\cW_\cK$. Conversely,
	a consistent world $w$ defines a set of repairs $Rep_w(\cK) =\{\cR|$ $\cR$ is a repair, $w\subseteq\cR\cup\cT\}$.
	The next lemma follows from these definitions.
	\begin{lem}
		\label{th:w-to-r}
		Given a KB $\cK$ and a consistent world $w\subseteq \cK$ that contains all axioms from the TBox of \cK, then the set $Rep_w(\cK)$ is not empty.
	\end{lem}
	\begin{proof}
		Each world $w$ contains all the axioms from the TBox and some axioms from the ABox. If a world is consistent then we can add ABox axioms to it until  the set of ABox axioms becomes equivalent to a repair, i.e., it contains all the assertions contained in the repair. So given a consistent world, this represents one or more repairs.
		
		By refutation, suppose there exists a consistent world $w$ such that $Rep_w(\cK)$ is empty. This means that there exists a set of axioms $\cE\subseteq w$ such that it is not contained in any repair. By definition of repairs, if $\cE\not\subseteq\cR$ for every repair $\cR$, then $\cE$ is TBox-inconsistent. So, $w$ cannot be consistent.
	\end{proof}

	It is also important to note that: (1) a repair can be obtained by removing a hitting set of the justifications for the inconsistency; and (2) a consistent world that contains all TBox axioms is a subset of a repair.

	Given these definitions, we can state the following theorems.

	\begin{thm}
		\label{th:brave}
		Given a possibly inconsistent KB $\cK$, a query $Q$, and the Boolean formula $f_{\textsc{all-just}(\mathit{Incons},\cK)}$, $Q$ is true under the \emph{Brave semantics}, i.e., $\cK\models_{Brave}Q$ iff there exists at least one justification for $Q$ w.r.t. $\cK$ such that the corresponding Boolean formula $\phi$ joined with $\neg f_{\textsc{all-just}(\mathit{Incons},\cK)}$ is satisfiable.
	\end{thm}
	\begin{proof}
		From Section~\ref{sec:prob-descr-logics}, a justification identifies a set of worlds. The smallest one is the world containing the axioms in the justification. The set of justifications $\textsc{all-just}(\mathit{Incons},\cK)$ is represented by the formula $f_{\textsc{all-just}(\mathit{Incons},\cK)}$. Given a satisfying truth assignment of the variables of $f_{\textsc{all-just}(\mathit{Incons},\cK)}$, it is possible to build a corresponding set of axioms
		$\cE$ consisting of all those axioms $E_i$ whose corresponding variable $X_i$ is given value true in the assignment. These axioms are a justification for the inconsistency of the KB. Therefore, they define a set of worlds such that each world is inconsistent.	
		Analogously, from a satisfying truth assignment of the variables of the formula $\neg f_{\textsc{all-just}(\mathit{Incons},\cK)}$ it is possible to obtain a consistent world. Therefore, if there is at least one justification for $Q$ w.r.t. $\cK$ with $\phi$ its representation as Boolean formula, from a satisfying truth assignment of the variables of $\phi\wedge\neg f_{\textsc{all-just}(\mathit{Incons},\cK)}$ it is possible to create at least one consistent world w.r.t. the TBox of \cK where $Q$ is true. From Lemma~\ref{th:w-to-r} there exists at least one repair $\cR$ such that $(\cR,\cT)\models Q$. So, $\cK\models_{Brave}Q$.  Basically, this consists in finding all the \emph{consistent worlds} and checking if $Q$ holds in at least one of these worlds.
		
		If the formula $\phi \wedge \neg f_{\textsc{all-just}(\mathit{Incons},\cK)}$ is not true for every $\phi$,  it is not possible to find a consistent world where the query $Q$ is true. Hence, $Q$ is false in every consistent world and thus, from Lemma~\ref{th:w-to-r}, it is so in every repair. So, $\cK\not\models_{Brave}Q$.
	\end{proof}
	
	To compare our semantics with the AR semantics, we first need to give three more definitions~\cite{DBLP:journals/jair/BienvenuBG19}.
	\begin{defi}[Conflict~\cite{DBLP:journals/jair/BienvenuBG19}]
		A \emph{conflict} of $\cK = (\cA,\cT)$ is an inclusion-minimal subset of $\cA$ that is inconsistent together with $\cT$. The set of conflicts of $\cK$ is denoted $\mathit{confl}(\cK)$.
	\end{defi}
	Each conflict corresponds to a justification for the inconsistency that contains\linebreak the assertions of the conflict and some axioms from $\cT$, i.e., given a justification $j\in$\linebreak$\textsc{all-just}(\mathit{Incons},\cK)$, if we remove from the justification the terminological axioms we obtain the conflict $\{E_i | E_i \in j\cap\cA\}$.
	
	\begin{defi}[Set of conflicts of a set of assertions]
		Given $\cK = (\cA,\cT)$, the \emph{set of conflicts of a set of assertions} $\cC\subseteq\cA$, denoted as $\mathit{confl}(\cC,\cK)$ is
		\begin{align}
			\mathit{confl}(\cC,\cK)=&\{\cB|\cB\in\mathit{confl(\cK)},\cC\cap\cB\neq\emptyset\}\notag
		\end{align}
	\end{defi}
	
	\begin{defi}[Cause~\cite{DBLP:journals/jair/BienvenuBG19}]
		A \emph{cause} for a Boolean query $Q$ in a KB $\cK = (\cA,\cT)$ is an inclusion minimal subset $\cC\subseteq\cA$ consistent with $\cT$ such that $(\cC,\cT) \models Q$. We use $\mathit{causes}(Q,\cK)$ to refer to the set of causes for $Q$ in \cK.
	\end{defi}
	Given a query $Q$, each cause corresponds to a justification for the query $Q$ that contains the assertions of the cause and some axioms from $\cT$, i.e., given a justification $j\in\textsc{all-just}(Q,\cK)$, if we remove from the justification the terminological axioms we obtain the cause $\mathit{cause}(j)=\{E_i | E_i\in j\cap\cA\}$. Thus, from the set $\textsc{all-just}(Q,\cK)$ we can define the set of causes of a set of justifications.
	\begin{defi}[Set of causes of the set of justifications]
		Given $\cK = (\cA,\cT)$, a query $Q$ and its set of justifications $\textsc{all-just}(Q,\cK)$, a \emph{set of causes of the set of justifications} $\mathit{cause}(\textsc{all-just}(Q,\cK))$ is defined as
		$$\mathit{cause}(\textsc{all-just}(Q,\cK)) = \{\mathit{cause}(j)|j\in\textsc{all-just}(Q,\cK)\}$$
	\end{defi}
	
	\begin{thm}
		\label{th:ar}
		This theorem derives from Theorem 4.11 and Remark 4.12 of \cite{DBLP:journals/jair/BienvenuBG19}. Given a possibly inconsistent KB $\cK$, a query $Q$, and the Boolean formulae $f_{\textsc{all-just}(\mathit{Incons},\cK)}$ and $f_{\textsc{all-just}(Q,\cK)}$,	
		$Q$ is not true under the \emph{AR semantics}, i.e., $\cK\not\models_{AR}Q$ iff $f_{\neg \textsc{all-just}(Q,\cK)}\wedge \neg f_{\textsc{all-just}(\mathit{Incons},\cK)}$ is satisfiable, where the formula $f_{\neg \textsc{all-just}(Q,\cK)}$ is defined as 
		$$f_{\neg \textsc{all-just}(Q,\cK)} = f_{\neg \textsc{all-just}(Q,\cK)}^1\bigwedge f_{\neg \textsc{all-just}(Q,\cK)}^2$$
		where, with $X_{\cC,\cB}$ new Boolean variables representing the different ways of contradicting $\cC$,
		$$f_{\neg \textsc{all-just}(Q,\cK)}^1=
		\bigwedge_{\cC\in\mathit{causes}(Q,\cK)}\bigvee_{\cB\in\mathit{confl}(\cC,\cK)}X_{\cC,\cB}$$
		and, with $\mathit{vars}(f)$ the set of variables appearing in $f$,
		$$f_{\neg \textsc{all-just}(Q,\cK)}^2=\bigwedge_{X_{\cC,\cB}\in\mathit{vars}(f_{\neg \textsc{all-just}(Q,\cK)}^1)}\bigwedge_{\beta\in\cB\setminus\cC}\neg X_{\cC,\cB}\vee X_\beta$$
		
	\end{thm}
	
	\begin{proof}
		The proof follows the same steps of \cite[Theorem 4.11 and Remark 4.12]{DBLP:journals/jair/BienvenuBG19}, $\neg f_{\textsc{all-just}(\mathit{Incons},\cK)}$ represents the set of consistent worlds,  and so the set of repairs, while $f_{\neg \textsc{all-just}(Q,\cK)}$ the set of repairs where $Q$ is not true. It is built in two steps: the first builds the formula $f_{\neg \textsc{all-just}(Q,\cK)}^1$ representing that every cause is contradicted, and the second builds the formula $f_{\neg \textsc{all-just}(Q,\cK)}^2$ ensuring that when a cause is contradicted, every axiom of the conflict not belonging to $\cC$ is present. If the formula $f_{\neg\textsc{all-just}(Q,\cK)}\wedge\neg f_{\textsc{all-just}(\mathit{Incons},\cK)}$ has an assignment of the Boolean variables that satisfies it, then there is a repair where the query is not true.
		
		If the formula $f_{\neg\textsc{all-just}(Q,\cK)}\wedge\neg f_{\textsc{all-just}(\mathit{Incons},\cK)}$ is not satisfiable, then it is not possible to find any consistent  world contradicting the formula. Hence, $Q$ is true in every consistent world and thus, from Lemma~\ref{th:w-to-r}, it is so in every repair. So, $\cK\models_{AR}Q$.
	\end{proof}
	
	\begin{thm}
		\label{th:iar}
		Given a possibly inconsistent KB $\cK$, a query $Q$, and the set $\cE$ of all the ABox axioms that appear in at least one justification for the inconsistency, $Q$ is true under the \emph{IAR semantics}, i.e., $\cK\models_{\mathit{IAR}}Q$, iff there exists at least one justification for $Q$ w.r.t. $\cK$ such that none of its axioms belongs to $\cE$.
	\end{thm}
	\begin{proof}
		By definition, the justifications for the inconsistency of a KB $\cK$ tell which combinations of axioms cause the inconsistency. Speaking in terms of repairs, given the set of the justifications for the inconsistency, each justification defines a conflict. For the definition of repairs, every repair will be defined as the KB $\cK$ to which at least one of the axioms in each justification for the inconsistency have been removed. Let us define the intersection of all the repairs as $\cD=\bigcap_{\cR\in Rep(\cK)}\cR=\cK\setminus\cE$, where $\cE$ is the set of all axioms contained in the justifications for the inconsistency. For $Q$ to be true in $\cD$ it is necessary that there is at least one justification $j$ for $Q$ w.r.t. $\cK$ such that $j\cap\cE=\emptyset$.
		
			Suppose that every justification $j_i$ for $Q$ in \cK contains at least one axiom from \cE, i.e., $\forall i:j_{i}\cap\cE\not=\emptyset$. Pick an axiom $\alpha\in j_{i}\cap\cE$,
			this axiom cannot be in the intersection of the repairs $\cD$, so $j_{i}$ is no longer a justification for $Q$ in \cD. 
			Since this is true for all justifications for $Q$, the query has no justification in the intersection of the repairs, so $\cK\not\models_{\mathit{IAR}} Q$
		\end{proof}

		It is worth noting that our approach, differently from the Brave/AR/IAR repair semantics, allows to more finely tailor the KB to the domain by choosing which axioms can be removed from the KB to build a repair and which axioms cannot. If we consider all the TBox axioms as certain, we restrict ourselves to the case of fixed TBox, if we add a probability value to all the axioms of the TBox, we consider the case where the TBox can also be repaired, while if we add a probability to some axioms only we work in an intermediate setting. This approach shares similarities with the Generalized Repair (GR) semantics~\cite{10.5555/3032027.3032070}, where a KB is split into two parts, a soft part containing axioms (actually facts or tuple-generating dependencies and negative constraints in the paper) that can be removed to build repairs, and a hard part that cannot be removed. A generalized repair is a KB containing the hard part and a maximal subset of axioms of the soft part such that the KB is consistent and the addition of any other axiom from the soft part makes the KB inconsistent. A query holds under the GR semantics if it holds in every generalized repair, similarly to the AR semantics. By using DISPONTE, we use probabilities to choose which axioms can be removed from the KB to repair it. In the sense of the GR semantics, we use probability to discriminate hard from soft axioms. The paper presenting the GR semantics also proposed the Local Generalized Repair (LGR) semantics where the tuple-generating dependencies (TGDs) and negative constraints (NCs) are considered ground when creating repairs. In this way, to build a repair, it is possible to remove both fact and ground TGDs/NCs from the KB without removing the whole set of TGD/NC. This approach cannot be applied by directly using DISPONTE.
		
		Irrespectively of whether the axioms are probabilistic or not, Theorems~\ref{th:brave}, \ref{th:ar}, and \ref{th:iar} still hold. Suppose that we consider every axiom of the KB as possible cause of inconsistency, if an axiom does not appear in any justification for the inconsistency, then it appears in every repair. On the other hand, if an axiom is contained in a justification for the inconsistency, then there will be at least one repair that does not contain it. Since in our proofs we consider only the axioms in the justifications, they hold irrespectively of the types of axioms considered.

		Moreover, the results using the Boolean formulae can be checked by means of BDDs. For example, given the set of justifications for the inconsistency, we can find the Boolean formula $f_{\textsc{all-just}(\mathit{Incons},\cK)}$ representing it and compile the formula into a BDD $BDDI$. To find $\neg f_{\textsc{all-just}(\mathit{Incons},\cK)}$, it is sufficient to negate  $BDDI$. Similarly, to see if a formula is not satisfiable, once it is compiled into a BDD, it is sufficient to see whether the BDD is equal to the $0$ leaf. We can use the BDDs of the justifications to compile the formulas in the theorems of this section.
		Therefore, our approach features some desirable characteristics: it can handle every DL language equipped with a reasoner able to return justifications; it provides two different tools able to reason under the repair semantics whilst computing the justifications and the probabilities of queries and inconsistency, making the results more informative; it directly works on DL KBs (it uses DL constructs such as annotations), without resorting to DBMS or converting to different languages.
		
		\section{Related Work}
		\label{sec:related}

		There are several lines of work on the subject of reasoning in the case of inconsistency. In this section we will first analyse approaches that consider certain KBs. Next, we will consider approaches that exploit uncertainty, grouping them into proposals based on probabilistic and possibilistic semantics. Finally, we will consider approaches that use weighted information to guide the reasoning process with inconsistent KBs.
		
		\subsection{Certain Inconsistency Management}
		For example, the standard DL semantics can be extended by the definition of a Four-Valued Logic, where classical implication is used together with two other types of implication of greater strength~\cite{DBLP:conf/dlog/MaHL08}.
		Their definition is given in terms of projections of concepts, i.e., every concept in a KB is associated with two sets of domain elements containing those known to belong to the concept and those known to belong to its complement. These two sets can overlap. Each implication can be translated into classical DL sub-class implications in a pre-processing step of the KB. 
		Another possible semantics introduces different types of negations, as in~\cite{DBLP:journals/ijar/ZhangXLB14}, where two types of negation are considered to handle information known to be false and information that is not known to be true. However, these approaches force the developers of the KBs to distinguish the different versions of implications or negations, which may be not intuitive for those who are not expert of logics. Changing the standard syntax and semantics can also bring compatibility issues with other KBs and the impossibility of using standard reasoners.

		Another approach considers the \emph{repairs} of an inconsistent KB. As already discussed, the repairs are parts (inclusion-maximal subsets) of the assertional axioms of the (inconsistent) KB that are consistent with the terminological axioms. These represent the possible ways of repairing an inconsistent KB by preserving as much information as possible (in the sense of set inclusion) to obtain a consistent KB. There could be many different repairs, depending on how many assertional axioms cause the inconsistency. There are several ways to build repairs, e.g., Baader et al. \cite{baader2021computing} look for optimal repairs where the least number of consequences is removed w.r.t. $\mathcal{EL}$ KBs, and several semantics based on the repairs, making the inference process tolerant to inconsistency. A query is true w.r.t. the inconsistent KB if, 
		for example, it is true in every repair of the KB (AR semantics~\cite{DBLP:conf/rr/LemboLRRS10}), in the intersection of the repairs (IAR semantics~\cite{DBLP:conf/rr/LemboLRRS10}), or in at least one repair (Brave semantics~\cite{DBLP:conf/ijcai/BienvenuR13}).
		A comprehensive introduction to repairs can be found in~\cite{DBLP:conf/rweb/BienvenuB16}. 
		
		One of the most prominent ways to answer queries under the AR semantics reduces the problem to SAT~\cite{DBLP:journals/jair/BienvenuBG19,DBLP:conf/kr/BienvenuB22}. In particular, ORBITS~\cite{DBLP:conf/kr/BienvenuB22} exploits justifications to build a conflict graph and sets of possible answers for the query that are analysed by a SAT solver to check Brave, AR or IAR semantics. A possibility of avoiding the use of a SAT solver and building the repairs is to iteratively select a subset of the (inconsistent) KB until the subset entails the query~\cite{fang2010reasoning}. This approach is comparable to the Brave semantics. Ludwig and Pe{\~{n}}aloza~\cite{DBLP:conf/jelia/LudwigP14} proposed to precompile and save all the repairs (possibly exponentially many) to answer subsumption queries w.r.t. TBox in $\mathcal{EL}$, thus without considering individuals in the ABox. Answering queries under Brave, AR, IAR and other semantics is efficient if all the repairs are precompiled. They also proposed a second, and more efficient, way: labelling every subsumption axiom with the set of repairs that contain it. This requires polynomial space without effectively building the repairs, which can be read using a directed acyclic graph.
		Other possible repair semantics are generalized and locally generalized repair semantics~\cite{10.5555/3032027.3032070}, where, as discussed in the previous section, it is possible to define each fact as soft or hard, either tuple-generating dependency (TGD) and negative constraint (NC) for generalized repairs, and ground TGD and ground NC for locally generalized repairs. A repair contains the hard part of the KB and a maximal set of soft facts, and (ground) TGDs and NCs. A query is true under these two semantics if it is true in every repair.
		
		Other approaches translate the KB into logic programming clauses~\cite{DBLP:conf/aike/FiorentinoGMT19,DBLP:journals/aai/GomezCS10,DBLP:conf/aaai/DuWS15}, performing inference through argumentation or abduction. 
		
		\subsection{Probabilistic Inconsistency Management}
		Even though uncertainty is pervasive in real world domains, there is a lack of proposals that exploit probabilistic information. 
		A preliminary idea was presented in \cite{DBLP:conf/ecai/CeylanLP16,DBLP:conf/rweb/CeylanL18} where the authors consider probabilistic Datalog$\pm$ ontologies and repairs that can be built by removing terminological axioms
		from the original KB as well, i.e., the inconsistency may also come from TBox axioms. They consider the possible world semantics, which corresponds to DISPONTE, provide two approaches to compute the probability of queries and complexity results but not a system for performing such type of inference. The first approach, differently from ours, considers only consistent worlds. This is done by assigning probability 0 to all the inconsistent worlds and re-normalizing the probabilities of consistent worlds by putting all the probability mass on consistent worlds. Then, the probability of a query is the sum of the probabilities of the worlds where the query holds. It is important to note that this semantics assumes that the error is in the probability distribution and modifies it to ``remove'' this error. However, as argued by the authors, this may lead to unexpected results. As a possible example, consider an inconsistent KB containing the axiom $0.6::a:C$ not related to axioms causing the inconsistency, and this is the only axiom in the whole KB that ensures that $a$ belongs to $C$. If we ask for the probability of $a$ to belong to $C$, we expect to obtain $0.6$. However, with this semantics we may obtain a probability larger than $0.6$. For this reason, this semantics has been updated in \cite{DBLP:conf/ecai/CeylanLP16} and further analysed in \cite{DBLP:conf/rweb/CeylanL18} to consider generalized repairs. The probability of the query is computed as the sum of the probabilities of the worlds in which the query is true in every generalized repair that can be built from the world. This second approach is closely related to DISPONTE as we are based on a similar probabilistic semantics. The main difference is in the way the probability is computed: we do not need to find every world and every repair of the consistent worlds. Moreover, we propose different implementations able to compute the probability.
		
		A similar approach from the same authors \cite{DBLP:conf/dlog/CeylanP14} associates a KB with a Bayesian network with variables $V$. 
		Axioms take the form $E:X=x$ where $E$ is a DL axiom and $X=x$ is an annotation with $X\subseteq V$ and $x$ a set of values for these variables.
		The Bayesian network assigns a probability to every assignment of $V$, called a world. Then, the probability of a query $Q=E:X=x$ is given by the sum of the probabilities of the worlds where $X=x$ is satisfied and where $E$ is a logical consequence of the theory composed of the axioms whose annotation is true in the world.
		This approach was applied to extend the $\mathcal{EL}$ DL in \cite{DBLP:conf/sum/CeylanP15,DBLP:conf/ore/CeylanMP15,DBLP:journals/jar/CeylanP17}, defining the probabilistic DL called $\mathcal{BEL}$. 
		The system BORN \cite{DBLP:conf/ore/CeylanMP15} answers probabilistic subsumption queries w.r.t. $\mathcal{BEL}$ KBs. It exploits ProbLog \cite{DBLP:conf/ijcai/RaedtKT07} 
		for managing the probabilistic part of the KB. 
		DISPONTE is a special case of these semantics where 
		every axiom $E_i:X_i=x_i$ is such that $X_i$ is a single Boolean 
		variable and the Bayesian network has no edges, i.e., all the variables
		are independent. This is an important special case that greatly simplifies
		reasoning, as computing the probability of the worlds takes a time linear
		in the number of variables.
		However, in case the added expressiveness of these formalisms is needed,
		the Bayesian network could be translated into an equivalent one with only 
		mutually unconditionally independent random variables. Moreover, under this semantics, if a KB is inconsistent, then by definition all entailments hold with probability~1~\cite{DBLP:journals/jar/CeylanP17}.
		
		Other approaches consider completely different logics, e.g., in~\cite{DBLP:conf/dali/BilkovaFMN20} the authors propose a probabilistic logic based on the probabilistic extensions of the Belnap-Dunn logic~\cite{belnap2019computer} combined with a bilattice logic, or an extension of Lukasiewicz's logic~\cite{DBLP:journals/jair/Lukasiewicz99}, defining a two-layer modal logical framework to account for belief in agent systems working on possibly inconsistent probabilistic information. However, these logics are different from DLs, thus it is not possible to make a meaningful comparison and they cannot be applied to the KBs connected in the open linked data cloud. Potyka and Thimm defined reasoning on linear probabilistic KBs \cite{DBLP:journals/ijar/PotykaT17}, where the uncertain knowledge in the KB is represented as linear probabilistic constraints, covering different logical formalisms such as Nilsson's logic~\cite{DBLP:journals/ai/Nilsson86} or Lukasiewicz's logic~\cite{DBLP:journals/jair/Lukasiewicz99}. In this approach, a so called \emph{generalized} model is used to allow probabilistic entailment on inconsistent KBs. A generalized model is defined as a probability distribution that minimally violates the KB. So, the probability values associated with the constraints must be carefully chosen in order not to violate the probability distribution. Moreover, since the violation is computed on the constraints, a second KB must be added that contains a consistent set of integrity constraints. Differently, DISPONTE considers each axiom as independent so the assignment of a probability value can be done independently for each axiom. Koch and Olteanu~\cite{DBLP:journals/pvldb/KochO08} consider probabilistic databases and apply a \emph{conditioning} operation that removes possible worlds not fulfilling given conditions, where conditions are a kind of constraints. This is similar to building repairs. Then, the query is asked w.r.t. the reduced database, and a confidence computation will return Bayesian conditional probabilities w.r.t. the original database.  This probability computation is similar to that of Ceylan \textit{et al.} \cite{DBLP:conf/ecai/CeylanLP16,DBLP:conf/rweb/CeylanL18} with the difference that the probabilities of consistent worlds are not modified when removing the inconsistent worlds from the probability computation. 
		Tammet et al.~\cite{DBLP:conf/cade/TammetDJ21} consider FOL, a semantics based on degrees of belief, and present an algorithm with various steps: they first calculate the decreasing confidence by means of a modified resolution search collecting different answers and justifications. Next the justifications are combined by means of the cumulation operation. Finally, the algorithm collects negative evidence for all the answers obtained so far, separately for each individual answer. This search is also split into resolution and cumulation. The search may not terminate, so a time limit is imposed, and the answer may not be complete. While this approach uses justifications, the main difference with our approach is that, differently from DISPONTE, the degrees of belief used by Tammet et al. are treated as confidences and not as probability values in the computation of the confidence of the query. Therefore, it does not define a probability distribution, so the results can take values larger than 1.0.
		Both Koch and Olteanu, and Tammet et al.'s proposals share similarities with our approach, the first in the use of the probabilistic semantics, the latter in the search of two sets of justifications.

		\subsection{Weighted Inconsistency Management}
		Another  possible solution for managing inconsistent KBs is to use different types of weights to guide the inference, mainly defining some measure of priority or importance on the information in the KB.
		Some approaches propose the use of priority levels or weights \cite{DBLP:conf/aaai/BienvenuBG14} to select repairs with higher priority/weight first or use weights to stratify the KBs  and build sub-ontologies by keeping as many axioms with higher weights as possible \cite{DBLP:conf/rweb/QiD12}.  To facilitate the parameter assignment task, one can exploit modularization-based approaches (such as~\cite{DBLP:conf/aswc/SuntisrivarapornQJH08}) to find local modules and assign the parameters considering only axioms in individual modules. More in general, another way of facilitating the work of knowledge engineers, especially in very large KBs, could be to consider the reliability of data source, associating a confidence level to each data source and use this value as a weight for the information extracted from it. For example, we could extract axioms from, e.g., DBPedia and SNOMED CT, about a certain disease and associate to all the axioms extracted from SNOMED CT a higher probability than those extracted from DBPedia, as we consider the first KB more reliable. This approach to associate weights/probability values to axioms can also be used with the majority of the approaches presented in this section, with DISPONTE among them, which however can also exploit EDGE~\cite{RigBelLamZese13-RR13a-IC,CotZesBel15-ILP-IC} or LEAP~\cite{RigBel14-URSWa-BC,CotZes15-AIIADC-IW}, two tools for automatically learning parameters of a DISPONTE KB by exploiting data contained in it.
		
		\subsection{Inconsistency Measures}
		From a different point of view, the use of probability to measure the inconsistency of a KB is related to the plethora of inconsistency measures defined in literature~\cite{thimm2018evaluation,DBLP:journals/ai/BesnardG20}. An \emph{inconsistency measure}~\cite{DBLP:journals/ndjfl/Grant78} assesses the severity of inconsistency, it gives the amount of inconsistency in a knowledge base, possibly with respect to the size of the KB. Usually, measures that fall in the first case are called \emph{absolute}, differently from the latter, called \emph{relative}, that consider, e.g., the number of assertions in the ABox, or the number of axioms in the whole KB. These measures can help in handling inconsistency. However,  justifications are needed to debug the KB.
		In the last decades, many different measures have been proposed \cite{DBLP:conf/kr/HunterK08,DBLP:conf/ijcai/GrantH11,DBLP:conf/ecai/XiaoM12,DBLP:conf/ecsqaru/GrantH13,DBLP:conf/ijcai/KoniecznyLM03}.
		Others consider the KB as a set of consistent logic formulae (inconsistent formulae can be split into consistent sub-formulae). Each consistent formula has at least one model, which is a world, which can be represented by a point in a Euclidean space. Using these points, it is possible to define measures based on the distance between the points in the Euclidean space~\cite{DBLP:conf/ecsqaru/GrantH13}.
		
		In a broad sense, computing the probability of the inconsistency of a KB is equivalent to computing an inconsistency measure that depends on the degree of belief in the axioms of the KB. From this point of view, the resulting measure may be considered as something in between of an absolute and a relative measure, because it does not consider the size of the KB, but the probabilities specified in it. De Bona et al.~\cite{DBLP:journals/jair/BonaGHK19} provide a classification of inconsistency using a bipartite graph relating logical formulae (built on axioms) and minimal inconsistent subsets of the KB. This allows one to compute different measures based on the count of formulae or subsets of the axioms of the KB from the bipartite graph. However, a direct comparison among all these measures is not possible in the majority of cases, because each  has its
		rationale and grows differently as the KB increases.

	\section{Experiments}
	\label{sec:exp}
	To test the feasibility of our approach, we extended our reasoners TRILL~\cite{DBLP:journals/tplp/ZeseCLBR19,ZesBelRig16-AMAI-IJ,DBLP:journals/ws/ZeseC21} and BUNDLE~\cite{RigBelLamZese13-RR13b-IC,CotRigZesBelLam-IC-2018} as described in the previous sections. TRILL is written in Prolog and can handle \shiq KBs, while BUNDLE is written in Java and uses an underlying non-probabilistic OWL reasoner to collect the set of justifications w.r.t. OWL 2 (essentially \sroiqd) KBs. In particular, BUNDLE embeds Pellet~\cite{DBLP:journals/ws/SirinPGKK07}, Hermit~\cite{shearer2008hermit},  FaCT++~\cite{tsarkov2006fact++} and JFaCT\footnote{\url{http://jfact.sourceforge.net/}} as OWL reasoners, and three justification generators, namely GlassBox (only for Pellet), BlackBox and OWL Explanation. Moreover, BUNDLE also encapsulates TRILL.
	
	In this paper we refer to the extension of the reasoners TRILL and BUNDLE as \textit{TRILLInc} and \textit{BUNDLEInc} respectively. In their original version, both reasoners, in case of an inconsistent KB, can be used only to collect the justifications for the inconsistency of the KB, while their extension apply the notions explained in this paper.
	
	As regards TRILL and \textit{TRILLInc}, to the almost 8000 lines of code of TRILL, we needed to add only 144 lines of code and to modify another hundred or so lines of code to implement the computation of $P_C(Q)$. To implement the computation of the repair semantics we needed to add approximatively 170 lines of code. \textit{TRILLInc} implements the extension of the tracing function.
	The code of \textit{TRILLInc}, together with the KBs used in this section, is available on GitHub\footnote{\url{https://github.com/rzese/trill\_inc}}.

	As regards BUNDLE and \textit{BUNDLEInc}, we adapted the code so that all the internal reasoners can be used to solve queries w.r.t. possibly inconsistent KBs. The code of \textit{BUNDLEInc} implements the addition of the negation of the query in the KB by means of the fresh concept $C_{Q_P}$. \textit{BUNDLEInc} can use all the reasoners used by BUNDLE (default is Hermit) with the ``OWLExplanation'' justification generator. \textit{BUNDLEInc} can only use this justification generator because other generators, such as the ``GlassBox'' generator developed by the authors of Pellet, may present some problems when the KB is inconsistent. These issues are present in all those generators that exploit the fact that the KB is consistent to avoid searching the entire KB for justifications. \textit{BUNDLEInc} is available from BitBucket\footnote{\url{https://bitbucket.org/machinelearningunife/bundle/src/bundle_inc/}}.

	To test the approach we presented in this paper, we pose ourselves three questions:
	\begin{description}
		\item[Q1] How much do the extensions affect the running time?
		\item[Q2] How does our approach compare with CQApri~\cite{DBLP:journals/jair/BienvenuBG19}, a well-known reasoner for the repair semantics?
		\item[Q3] How does our approach scale when considering complex KBs?
	\end{description}
	
	To answer these questions, we conducted different tests. Firstly, we ran our systems in their two versions, the classical one and the ``\textit{Inc}'' one (the version implementing the approach proposed in this paper). Then, we concentrate on one of our systems, \textit{BUNDLEInc}, and compare it with CQApri\footnote{Available at \url{https://lahdak.lri.fr/CQAPri}} on the benchmark its authors used to test the system. This choice depends on two facts: first, since CQApri is implemented in Java, considering \textit{BUNDLEInc} removes as much as possible the dependence from the test platform; second, \textit{TRILLInc} needed more than 12 hours for the KBs considered in this test, as better discussed in section \ref{subsec:q2}.
	
	All the tests have been performed on a Linux machine with equipped with Intel\textsuperscript{\textcopyright} Core\textsuperscript{TM} i7-8565U CPU @ 1.80GHz, with 16 GiB of RAM.
	
	In this section, we call ``DISPONTE time'' the time for building the BDD and computing the probability from it. In particular, the time for the computation of the probability is negligible with respect the time for building the BDD. We indicate with ``Repair time'' the whole time to find out under which repair semantics the query is true: we check whether the query is true under the Brave semantics, if the test returns true, we check the IAR semantics and, in case the query is not entailed under this semantics, we check the AR semantics. Therefore, if the reasoner returns false for a query, it checks only the Brave semantics, if it returns true under IAR for the query, the reasoner checks the Brave and IAR semantics, while if the reasoner returns that the query is true under AR or Brave, it checks the Brave, IAR and AR semantics. It is not possible to clearly separate the time to check each semantics, since the BDDs created during the computation of the Brave semantics are used also for the IAR and AR semantics, while the operations done during the computation of the IAR semantics are also used for the AR semantics. However, the time for checking the IAR semantics is on average similar or smaller than that for checking the Brave semantics, while the time for the AR semantics, is on average 95\% of the whole computation time needed for the repair semantics.

	\subsection{Q1: How much do the extensions affect the running time?}\label{subsec:q1}
	For the first question, 
	we considered KBs similar of those presented in Test~3 of~\cite{DBLP:journals/tplp/ZeseCLBR19}. We built 9 KBs of increasing size containing the following axioms, for $i$ varying from 1 to $n$, with $n\in\{2,3,4,5,6,7,8,9,10\}$:
	$$\begin{array}{cccc}
		B_{i-1}\sqsubseteq P_i\sqcap Q_i&\>\>\>\>\>
		P_i\sqsubseteq B_i\>\>\>\>\>&
		Q_i\sqsubseteq B_i\>\>\>\>\>&
		x:B_0
	\end{array}$$
	where $B_i$, $Q_i$, $P_i$ are simple concepts. However, in principle, they can be concept expressions of any expressivity handled by the reasoner.
	These KBs have a number of justifications for the query $Q=x:B_n$ that grows exponentially with $n$. The choice of this KB is to force the creation of a number of justifications that grows exponentially, preferring larger number of justifications instead of bigger KBs. This choice has been made to stress the two most expensive operations in our reasoners: on the one hand, the application of tableau expansion rules, since it needs the management of choice points, backtracking and, possibly, creating new tableaux depending on the implementation of the reasoner; on the other hand, the management of the BDDs. Moreover, it is easy to decide where to put inconsistency and decide the percentage of probabilistic axioms.
	
	For each KB, we created a second version in which we added a disjoint-classes axiom asserting that classes $B_j$ and $B_k$ are disjoint, with $j,k$ set as explained below. This, combined with the class assertion axiom, makes the KBs inconsistent.
	We built a KB for each value of $n$ in order to see how the running time changes as $n$ increases.
	We run the query $Q=x:B_n$ 10 times w.r.t. each KB in order to compute the average running time for answering the query. We do not expect to have large variability in the running times among the 10 runs, but we decided to execute the query more times in order to make sure that the reported results are not affected by the load of the machine used during the tests. We compared the running time of the original version of the reasoners to solve $Q$ w.r.t. the KBs without the disjoint-classes axiom  causing the inconsistency, with the running time taken by the ``\textit{Inc}'' version of the reasoners w.r.t. four settings:
	\begin{description}
		\item[(S1)] KBs without the disjoint-classes axiom, thus consistent, to see how much overhead we add to the whole process and so, how much the introduced extension affects the reasoning in case of consistent KBs;
		\item[(S2)] KBs with the disjoint-classes axiom considering the classes $B_0$ and $B_1$ ($j=0,k=1$), thus inconsistent, where there are only two justifications for the inconsistency, consisting both of three axioms, to see how the running time changes in the best case, i.e., when collecting justifications for the inconsistency is trivial;
		\item[(S3)] KBs with the disjoint-classes axiom considering the classes $B_{n}$ and $B_{n-1}$ ($j=n,k=n-1$), thus inconsistent, whose justifications  are the same of those of the query, to see how the running time changes in the case with the higher number of justification possible for the KB.  In the worst case, i.e., $n=10$, the KB contains 30 axioms but there are $2^n+2^n=2048$ justifications, 1024 for $Q$ and 1024 for the inconsistency, a situation difficult to achieve even with large KBs.
		\item[(S4)] KBs with the disjoint-classes axiom considering the classes $B_{n}$ and $B_{n-1}$ ($j=n,k=n-1$) (as in setting \textbf{(3)}) with the addition of four new axioms:
		$$\begin{array}{llll}
			C_0\sqsubseteq C_{0,1}\sqcap C_{0,2}&\>\>\>\>\>
			C_{0,1}\sqsubseteq C_1\>\>\>\>\>&
			C_{0,2}\sqsubseteq C_1\>\>\>\>\>&
			x:C_0
		\end{array}$$
		In this case the query is $Q=x:C_1$,in order to test the reasoners when the number of justifications for the query is fixed while the number of justifications for the inconsistency grows.
	\end{description}

	\noindent
	Considering these four settings, we tested them in three different scenarios:
	\begin{description}
		\item[(SC1.1)] all the settings considering KBs where only the assertional axioms are probabilistic;
		\item[(SC1.2)] all the settings considering KBs where all the axioms are probabilistic;
		\item[(SC1.3)] only setting \textbf{(S3)} with KBs where we increased the number of probabilistic axioms.
	\end{description}

	\paragraph{Scenario (SC1.1).} Table~\ref{tab:test-ext-res-t-sc1} shows, for each KB and for each setting, the average running time, the minimum and maximum running time, the time for computing the BDD and the probability of the query, and the time for computing the repair semantics, all in seconds, for TRILL and \textit{TRILLInc}. The table also shows the ratio between the running time of TRILL/\textit{TRILLInc}. Table~\ref{tab:test-ext-res-b-sc1} shows the same information for BUNDLE/\textit{BUNDLEInc}. The average ratios between the average running time of the plain and \textit{Inc} versions of the reasoners are also shown, computed by considering all the KBs and the KBs of sizes from 5 to 10, i.e., the KBs with a significant number of justifications and so KBs where the extension could affect the running time more. For setting \textbf{(S4)} the running time ratio is computed against the running time of TRILL or BUNDLE on the KB of setting \textbf{(S4)} without inconsistency for query $Q=x:C_1$.
	
	As expected, the relative standard deviation\footnote{Relative standard deviation is the measure of the ratio of the standard deviation to the mean in percentage. So, given the mean value $\mu$ and the standard deviation $\sigma$, the relative standard deviation is $\frac{\sigma}{\mu}\cdot100$.} of the whole running time across the 10 repetitions is low for all running times collected: it is below 1\% for TRILL, while it is close to 4\% for BUNDLE, so we do not report them in the tables for simplicity.  These values mean that, given a query and a KB, the reasoning time does not depend on the choices done by the reasoner during the exploration of the search space.

	\begin{table}
		\centering
		\caption{Average running time in seconds computed on 10 executions of the query with the original version of the reasoner TRILL and the extended version of the reasoner \textit{TRILLInc} in settings \textbf{(S1)}, \textbf{(S2)}, \textbf{(S3)}, and \textbf{(S4)} for scenario \textbf{(SC1.1)}, together with the minimum and maximum running time, and ratio of the two average running time measurements, the times for computing the DISPONTE probability, and for the computation of the repair semantics in seconds with the percentage of the total running time. The average of the ratios considering all and only sizes greater than 4 are also shown.\label{tab:test-ext-res-t-sc1}}	
		\tiny
		\begin{tabular}{c|c|c|c|c|c|c|c|c}
			& \textbf{Time} &                & \textbf{Min} & \textbf{Max} & \multicolumn{2}{c|}{\textbf{DISPONTE Time}} & \multicolumn{2}{c}{\textbf{Repair Time}} \\
			\textbf{i} & \textbf{(s)}  & \textbf{Ratio} & \textbf{(s)} & \textbf{(s)} & \textbf{\%} &         \textbf{(s)}          & \textbf{\%} &        \textbf{(s)}        \\ \hline\hline
			&                                                          \multicolumn{8}{c}{\textbf{TRILL}}                                                           \\ \hline
			2      &     0.003     &                &    0.003     &    0.003     &    5.643    &            0.00019            &                    \multicolumn{2}{r}{ } \\
			3      &     0.004     &                &    0.004     &    0.004     &   10.743    &            0.00045            &                    \multicolumn{2}{r}{ } \\
			4      &     0.007     &                &    0.007     &    0.007     &   13.839    &            0.00094            &                    \multicolumn{2}{r}{ } \\
			5      &     0.020     &                &    0.019     &    0.020     &   10.760    &            0.00210            &                    \multicolumn{2}{r}{ } \\
			6      &     0.093     &                &    0.092     &    0.093     &    5.071    &            0.00471            &                    \multicolumn{2}{r}{ } \\
			7      &     0.593     &                &    0.588     &    0.601     &    1.521    &            0.00902            &                    \multicolumn{2}{r}{ } \\
			8      &     4.268     &                &    4.247     &    4.300     &    0.993    &            0.04237            &                    \multicolumn{2}{r}{ } \\
			9      &    32.893     &                &    32.807    &    33.039    &    0.585    &            0.19247            &                    \multicolumn{2}{r}{ } \\
			10     &    264.699    &                &   263.975    &   265.661    &    0.250    &            0.66154            &                    \multicolumn{2}{r}{ } \\ \hline\hline
			&                                              \multicolumn{8}{c}{\textbf{\textit{TRILLInc} setting (S1)}}                                              \\ \hline
			2      &     0.004     &      1.06      &    0.004     &    0.004     &    5.548    &            0.00020            &    0.753    &          0.00003           \\
			3      &     0.005     &      1.11      &    0.005     &    0.006     &    9.703    &            0.00045            &    1.152    &          0.00005           \\
			4      &     0.008     &      1.18      &    0.008     &    0.009     &   11.882    &            0.00096            &    1.473    &          0.00012           \\
			5      &     0.024     &      1.21      &    0.023     &    0.029     &    9.910    &            0.00234            &    1.090    &          0.00026           \\
			6      &     0.109     &      1.17      &    0.108     &    0.110     &    4.586    &            0.00499            &    0.526    &          0.00057           \\
			7      &     0.689     &      1.16      &    0.686     &    0.692     &    1.651    &            0.01137            &    0.532    &          0.00366           \\
			8      &     4.955     &      1.16      &    4.917     &    5.002     &    0.893    &            0.04424            &    0.058    &          0.00286           \\
			9      &    38.291     &      1.16      &    38.039    &    38.687    &    0.523    &            0.20040            &    0.016    &          0.00625           \\
			10     &    306.773    &      1.16      &   305.592    &   308.832    &    0.225    &            0.68957            &    0.097    &          0.29667           \\ \hline
			\multicolumn{2}{r|}{ } &      1.15      & \multicolumn{6}{l}{\textbf{Avg on results of size 2 to 10}}                                                          \\
			\multicolumn{2}{r|}{ } &      1.17      & \multicolumn{6}{l}{\textbf{Avg on results of size 5 to 10}}                                                          \\ \hline\hline
			&                                              \multicolumn{8}{c}{\textbf{\textit{TRILLInc} setting (S2)}}                                              \\ \hline
			2      &     0.004     &      1.09      &    0.004     &    0.004     &    6.975    &            0.00026            &    0.936    &          0.00003           \\
			3      &     0.005     &      1.11      &    0.005     &    0.005     &   10.177    &            0.00048            &    1.208    &          0.00006           \\
			4      &     0.008     &      1.16      &    0.008     &    0.008     &   12.754    &            0.00100            &    1.472    &          0.00012           \\
			5      &     0.023     &      1.17      &    0.023     &    0.023     &    9.748    &            0.00223            &    1.100    &          0.00025           \\
			6      &     0.109     &      1.18      &    0.108     &    0.112     &    4.631    &            0.00506            &    0.530    &          0.00058           \\
			7      &     0.691     &      1.17      &    0.688     &    0.694     &    1.650    &            0.01140            &    0.529    &          0.00366           \\
			8      &     4.980     &      1.17      &    4.954     &    5.006     &    0.891    &            0.04439            &    0.426    &          0.02121           \\
			9      &    38.362     &      1.17      &    38.116    &    39.243    &    0.523    &            0.20049            &    0.016    &          0.06259           \\
			10     &    306.770    &      1.16      &   305.555    &   308.353    &    0.225    &            0.69054            &    0.097    &          0.29624           \\ \hline
			\multicolumn{2}{r|}{ } &      1.15      & \multicolumn{6}{l}{\textbf{Avg on results of size 2 to 10}}                                                          \\
			\multicolumn{2}{r|}{ } &      1.17      & \multicolumn{6}{l}{\textbf{Avg on results of size 5 to 10}}                                                          \\ \hline\hline
			&                                              \multicolumn{8}{c}{\textbf{\textit{TRILLInc} setting (S3)}}                                              \\ \hline
			2      &     0.004     &      1.15      &    0.004     &    0.004     &    9.152    &            0.00036            &    1.096    &          0.00004           \\
			3      &     0.006     &      1.32      &    0.006     &    0.006     &   14.561    &            0.00081            &    1.698    &          0.00009           \\
			4      &     0.012     &      1.83      &    0.012     &    0.013     &   15.473    &            0.00192            &    1.758    &          0.00022           \\
			5      &     0.045     &      2.33      &    0.045     &    0.046     &    9.841    &            0.00448            &    1.100    &          0.00050           \\
			6      &     0.248     &      2.67      &    0.247     &    0.250     &    4.182    &            0.01038            &    0.268    &          0.00064           \\
			7      &     1.597     &      2.69      &    1.586     &    1.630     &    2.786    &            0.04451            &    0.166    &          0.00265           \\
			8      &    11.298     &      2.65      &    11.240    &    11.385    &    0.784    &            0.08853            &    0.052    &          0.00592           \\
			9      &    86.743     &      2.64      &    86.477    &    87.436    &    0.297    &            0.25736            &    0.180    &          0.15574           \\
			10     &    684.984    &      2.59      &   682.504    &   687.596    &    0.122    &            0.83374            &    0.086    &          0.58919           \\ \hline
			\multicolumn{2}{r|}{ } &      2.21      & \multicolumn{6}{l}{\textbf{Avg on results of size 2 to 10}}                                                          \\
			\multicolumn{2}{r|}{ } &      2.59      & \multicolumn{6}{l}{\textbf{Avg on results of size 5 to 10}}                                                          \\ \hline\hline
			&                                              \multicolumn{8}{c}{\textbf{\textit{TRILLInc} setting (S4)}}                                              \\ \hline
			2      &     0.004     &      1.23      &    0.004     &    0.004     &    6.893    &            0.00029            &    1.016    &          0.00004           \\
			3      &     0.006     &      1.34      &    0.006     &    0.006     &    9.578    &            0.00054            &    1.312    &          0.00007           \\
			4      &     0.011     &      1.62      &    0.011     &    0.011     &   10.132    &            0.00112            &    1.342    &          0.00015           \\
			5      &     0.040     &      2.06      &    0.040     &    0.042     &    6.094    &            0.00245            &    0.784    &          0.00031           \\
			6      &     0.220     &      2.35      &    0.218     &    0.224     &    2.486    &            0.00546            &    0.315    &          0.00069           \\
			7      &     1.499     &      2.53      &    1.494     &    1.508     &    0.818    &            0.01226            &    1.501    &          0.02249           \\
			8      &    10.894     &      2.55      &    10.838    &    10.952    &    0.250    &            0.02728            &    0.362    &          0.03946           \\
			9      &    84.846     &      2.58      &    84.432    &    87.071    &    0.238    &            0.20210            &    0.088    &          0.07424           \\
			10     &    676.727    &      2.55      &   673.832    &   679.774    &    0.019    &            0.13118            &    0.088    &          0.59637           \\ \hline
			\multicolumn{2}{r|}{ } &      2.09      & \multicolumn{6}{l}{\textbf{Avg on results of size 2 to 10}}                                                          \\
			\multicolumn{2}{r|}{ } &      2.44      & \multicolumn{6}{l}{\textbf{Avg on results of size 5 to 10}}                                                          \\ \hline
		\end{tabular}
	\end{table}

	\begin{table}
		\centering
		\caption{Average running time in seconds computed on 10 executions of the query with the original version of the reasoner BUNDLE and the extended version of the reasoner \textit{BUNDLEInc} in settings \textbf{(S1)}, \textbf{(S2)}, \textbf{(S3)}, and \textbf{(S4)} for scenario \textbf{(SC1.1)}, together with the minimum and maximum running time, and ratio of the two average running time measurements, the times for computing the DISPONTE probability, and for repair semantics computation in seconds with the percentage of the total running time. The average of the ratios considering all and only sizes greater than 4 are also shown.\label{tab:test-ext-res-b-sc1}}	
		\tiny
		\begin{tabular}{c|c|c|c|c|c|c|c|c}
			& \textbf{Time} &                & \textbf{Min} & \textbf{Max} & \multicolumn{2}{c|}{\textbf{DISPONTE Time}} & \multicolumn{2}{c}{\textbf{Repair Time}} \\
			\textbf{i} & \textbf{(s)}  & \textbf{Ratio} & \textbf{(s)} & \textbf{(s)} & \textbf{\%} &         \textbf{(s)}          & \textbf{\%} &        \textbf{(s)}        \\ \hline\hline
			&                                                          \multicolumn{8}{c}{\textbf{BUNDLE}}                                                          \\ \hline
			2      &     0.854     &                &    0.810     &    0.890     &    5.200    &             0.044             &                    \multicolumn{2}{r}{ } \\
			3      &     1.306     &                &    1.227     &    1.371     &    3.692    &             0.048             &                    \multicolumn{2}{r}{ } \\
			4      &     2.089     &                &    1.884     &    2.213     &    2.566    &             0.054             &                    \multicolumn{2}{r}{ } \\
			5      &     3.314     &                &    2.819     &    3.628     &    1.473    &             0.049             &                    \multicolumn{2}{r}{ } \\
			6      &     5.555     &                &    4.795     &    5.995     &    0.884    &             0.049             &                    \multicolumn{2}{r}{ } \\
			7      &     8.819     &                &    7.645     &    9.727     &    0.535    &             0.047             &                    \multicolumn{2}{r}{ } \\
			8      &    15.798     &                &    14.853    &    17.070    &    0.339    &             0.053             &                    \multicolumn{2}{r}{ } \\
			9      &    26.573     &                &    25.890    &    27.566    &    0.282    &             0.075             &                    \multicolumn{2}{r}{ } \\
			10     &    48.043     &                &    47.167    &    48.886    &    0.198    &             0.095             &                    \multicolumn{2}{r}{ } \\ \hline\hline
			&                                             \multicolumn{8}{c}{\textbf{\textit{BUNDLEInc} setting (S1)}}                                              \\ \hline
			2      &     1.523     &      1.78      &    1.332     &    1.663     &    2.982    &             0.045             &    0.085    &           0.0013           \\
			3      &     2.725     &      2.09      &    2.464     &    3.141     &    1.538    &             0.042             &    0.037    &           0.0010           \\
			4      &     5.237     &      2.51      &    4.537     &    6.284     &    0.787    &             0.041             &    0.025    &           0.0013           \\
			5      &     9.980     &      3.01      &    9.507     &    10.667    &    0.434    &             0.043             &    0.018    &           0.0018           \\
			6      &    17.893     &      3.22      &    17.172    &    18.538    &    0.254    &             0.046             &    0.017    &           0.0030           \\
			7      &    34.799     &      3.95      &    33.731    &    35.383    &    0.155    &             0.054             &    0.016    &           0.0054           \\
			8      &    71.227     &      4.51      &    70.117    &    73.587    &    0.081    &             0.058             &    0.014    &           0.0100           \\
			9      &    147.463    &      5.55      &   144.970    &   150.335    &    0.045    &             0.067             &    0.018    &           0.0260           \\
			10     &    310.723    &      6.47      &   306.245    &   318.138    &    0.035    &             0.109             &    0.018    &           0.0545           \\ \hline
			\multicolumn{2}{r|}{ } &      3.68      & \multicolumn{6}{l}{\textbf{Avg on results of size 2 to 10}}                                                          \\
			\multicolumn{2}{r|}{ } &      4.45      & \multicolumn{6}{l}{\textbf{Avg on results of size 5 to 10}}                                                          \\ \hline\hline
			&                                             \multicolumn{8}{c}{\textbf{\textit{BUNDLEInc} setting (S2)}}                                              \\ \hline
			2      &     1.950     &      2.28      &    1.720     &    2.103     &    2.257    &             0.044             &    0.051    &           0.0010           \\
			3      &     3.041     &      2.33      &    2.718     &    3.535     &    1.555    &             0.047             &    0.036    &           0.0011           \\
			4      &     5.622     &      2.69      &    4.906     &    6.200     &    0.822    &             0.046             &    0.027    &           0.0015           \\
			5      &     9.736     &      2.94      &    8.772     &    10.562    &    0.486    &             0.047             &    0.020    &           0.0019           \\
			6      &    18.208     &      3.28      &    17.006    &    19.542    &    0.286    &             0.052             &    0.017    &           0.0031           \\
			7      &    35.324     &      4.01      &    34.515    &    36.351    &    0.150    &             0.053             &    0.016    &           0.0058           \\
			8      &    71.343     &      4.52      &    70.034    &    74.598    &    0.085    &             0.061             &    0.014    &           0.0103           \\
			9      &    147.523    &      5.55      &   144.802    &   150.047    &    0.048    &             0.070             &    0.017    &           0.0255           \\
			10     &    312.981    &      6.51      &   307.939    &   316.076    &    0.032    &             0.100             &    0.018    &           0.0579           \\ \hline
			\multicolumn{2}{r|}{ } &      3.79      & \multicolumn{6}{l}{\textbf{Avg on results of size 2 to 10}}                                                          \\
			\multicolumn{2}{r|}{ } &      4.47      & \multicolumn{6}{l}{\textbf{Avg on results of size 5 to 10}}                                                          \\ \hline\hline
			&                                             \multicolumn{8}{c}{\textbf{\textit{BUNDLEInc} setting (S3)}}                                              \\ \hline
			2      &     2.268     &      2.66      &    2.019     &    2.506     &    1.979    &            0.0449             &    0.048    &           0.0011           \\
			3      &     4.547     &      3.48      &    3.924     &    5.047     &    0.974    &            0.0443             &    0.026    &           0.0012           \\
			4      &     8.858     &      4.24      &    8.012     &    9.735     &    0.542    &             0.048             &    0.017    &           0.0015           \\
			5      &    16.568     &      5.00      &    15.723    &    17.711    &    0.300    &             0.050             &    0.017    &           0.0028           \\
			6      &    31.267     &      5.63      &    30.524    &    32.246    &    0.176    &             0.055             &    0.016    &           0.0051           \\
			7      &    64.048     &      7.26      &    62.711    &    65.814    &    0.084    &             0.054             &    0.013    &           0.0082           \\
			8      &    133.616    &      8.46      &   130.435    &   135.210    &    0.053    &             0.070             &    0.015    &           0.0202           \\
			9      &    288.895    &     10.87      &   283.933    &   293.431    &    0.036    &             0.105             &    0.019    &           0.0536           \\
			10     &    600.368    &     12.50      &   600.329    &   600.447    &    0.025    &             0.148             &    0.014    &           0.0830           \\ \hline
			\multicolumn{2}{r|}{ } &      6.68      & \multicolumn{6}{l}{\textbf{Avg on results of size 2 to 10}}                                                          \\
			\multicolumn{2}{r|}{ } &      8.29      & \multicolumn{6}{l}{\textbf{Avg on results of size 5 to 10}}                                                          \\ \hline\hline
			&                                             \multicolumn{8}{c}{\textbf{\textit{BUNDLEInc} setting (S4)}}                                              \\ \hline
			2      &     2.563     &      2.95      &    2.207     &    2.837     &    1.697    &             0.044             &    0.066    &           0.0017           \\
			3      &     3.897     &      3.07      &    3.386     &    4.410     &    1.144    &             0.045             &    0.049    &           0.0019           \\
			4      &     6.440     &      3.08      &    5.776     &    7.263     &    0.700    &             0.045             &    0.033    &           0.0021           \\
			5      &    12.394     &      3.78      &    10.645    &    13.527    &    0.362    &             0.045             &    0.023    &           0.0028           \\
			6      &    21.586     &      3.92      &    20.285    &    23.420    &    0.223    &             0.048             &    0.018    &           0.0039           \\
			7      &    41.858     &      4.66      &    41.079    &    44.417    &    0.132    &             0.055             &    0.015    &           0.0061           \\
			8      &    82.718     &      5.17      &    81.881    &    84.250    &    0.072    &             0.059             &    0.015    &           0.0123           \\
			9      &    168.093    &      6.27      &   165.568    &   170.733    &    0.044    &             0.074             &    0.017    &           0.0286           \\
			10     &    360.209    &      7.54      &   357.361    &   368.554    &    0.031    &             0.110             &    0.017    &           0.0615           \\ \hline
			\multicolumn{2}{r|}{ } &      4.49      & \multicolumn{6}{l}{\textbf{Avg on results of size 2 to 10}}                                                          \\
			\multicolumn{2}{r|}{ } &      5.22      & \multicolumn{6}{l}{\textbf{Avg on results of size 5 to 10}}                                                          \\ \hline
		\end{tabular}
	\end{table}

	As one can see, the proposed extension adds an overhead which lies between 6\% (Ratio 1.06, $i=2$, setting \textbf{(S1)}) and 169\% (Ratio 2.69, $i=7$, setting \textbf{(S3)}) for TRILL. For BUNDLE the overhead is larger, lying between 78\% (Ratio 1.78, $i=2$, setting \textbf{(S1)}) and 1150\% (Ratio 12.50, $i=10$, setting \textbf{(S3)}).  Given the time needed for DISPONTE and for the repair semantics check, this ratio only depends on the way the reasoner explores the search space of the justifications, as their time is negligible (less than 1\%) when the running time is longer than a second. This exploration method differs between \textit{TRILLInc} and \textit{BUNDLEInc}. However, DISPONTE time does not increase as fast as the running time, dropping under 1\% when the running time becomes larger than several seconds for both TRILL and BUNDLE. Similar results are obtained for the computation of the repair semantics, with repair time is always below 2\% of the whole running time for TRILL and below 1\% for BUNDLE.  When the number of justifications is small, the overhead of the computation of the repair and the DISPONTE probability on the reasoning time is mitigated by an initialization phase of the reasoner, which affects all the executions  and does not depend on the query or on the number of justifications for the query, but only on the operations needed by the reasoner to start the search for justifications (e.g., KB loading, initialization of the internal reasoner for \textit{BUNDLEInc}). When the number of justifications increases, this initialization phase becomes more and more negligible, while the time to compute the query answer using the justifications increases with their number. This is shown in columns ``Time DISPONTE \%'' and ``Time Repair \%'' proving that the effects of the computation of repairs and probability becomes smaller. 
	
	Overall, the overhead is dependent on the time for finding the justifications. In fact, when the search for justification is fast, the general ratio (column ``Ratio'') is high, for increasing values of $i$ ($i>5$) the ratio starts to slowly decrease for \textit{TRILLInc}, where the search for justifications is not much affected by the extension. For \textit{BUNDLEInc}, the search for justifications is more affected by the changes to the tracing function, also due to how the ``OWLExplanation'' generator performs the search to generate all the justifications. For small values of $i$, the overhead introduced by the computation of the DISPONTE probability and repair semantics is more significant than that of the search for justifications.  From settings \textbf{(S1)} and \textbf{(S2)} we can see that, in the case of consistent KBs or when the number of justifications for the inconsistency is small, the general overhead is similar, and the number of justifications for the inconsistency is the aspect that most affect the running time, as it can be seen from settings \textbf{(S3)} and \textbf{(S4)} which show similar results.  So, the extension does not significantly affect the running time on consistent or inconsistent KBs with few, small justifications for the inconsistency.

\paragraph{Scenario (SC1.2).} Table~\ref{tab:test-ext-res-t-sc2} shows, for each KB and for each setting, the average running time, the minimum and maximum running time, the DISPONTE time, and the repair time, all in seconds, for TRILL and \textit{TRILLInc}. The table also shows the ratio between the running time of TRILL and\textit{TRILLInc}. Average ratios between the average running time of the two versions of the reasoner are also shown, computed by considering all the KBs and the KBs of sizes from 5 to 10. Table~\ref{tab:test-ext-res-b-sc2} shows the same information for BUNDLE/\textit{BUNDLEInc}. Again, the relative standard deviation of the whole time across all 10 repetitions is below 1\% except for setting \textbf{(S3)}, where it increases to 1.4\% for $i=5$ and $i=6$. These values confirm that, once the KB and query are fixed, the running time should not significantly vary. If this happens, the motivation should be sought by considering the CPU load of the machine where the reasoner is running.

In this setting, for $i>6$, we stopped the execution of TRILL after 12 hours while BUNDLE could not manage such values of $i$ due to out-of-memory errors caused by the library for managing BDDs when computing the repair semantics. As in scenario \textbf{(SC1.1)}, for setting \textbf{(S4)} the running time ratio is computed against the running time of TRILL or BUNDLE w.r.t. the KB of setting \textbf{(S4)} without inconsistency for query $Q=x:C_1$.

\begin{table}
	\centering
	\caption{Average running time in seconds computed on 10 executions of the query with the original version of the reasoner TRILL and the extended version of the reasoner \textit{TRILLInc} in settings \textbf{(S1)}, \textbf{(S2)}, \textbf{(S3)}, and \textbf{(S4)} for scenario \textbf{(SC1.2)}, together with the minimum and maximum running time, and ratio of the two average running time measurements, the times for computing the DISPONTE probability, and for repair semantics computation in seconds with the percentage of the total running time. The average of the ratios considering all and only sizes greater than 4 are also shown. The values ``--'' means that the execution has been stopped due to timeout (set to 12 hours).\label{tab:test-ext-res-t-sc2}}	
	\tiny
	\begin{tabular}{c|c|c|c|c|c|c|c|c}
		& \textbf{Time} &                & \textbf{Min} & \textbf{Max} & \multicolumn{2}{c|}{\textbf{DISPONTE Time}} & \multicolumn{2}{c}{\textbf{Repair Time}} \\
		\textbf{i} & \textbf{(s)}  & \textbf{Ratio} & \textbf{(s)} & \textbf{(s)} & \textbf{\%} &         \textbf{(s)}          & \textbf{\%} &        \textbf{(s)}        \\ \hline\hline
		&                                                          \multicolumn{8}{c}{\textbf{TRILL}}                                                           \\ \hline
		2      &     0.004     &                &    0.003     &    0.004     &    9.117    &            0.0003             &          \multicolumn{2}{c}{ }           \\
		3      &     0.004     &                &    0.004     &    0.005     &   17.796    &            0.0008             &          \multicolumn{2}{c}{ }           \\
		4      &     0.007     &                &    0.007     &    0.008     &   25.756    &            0.0019             &          \multicolumn{2}{c}{ }           \\
		5      &     0.021     &                &    0.021     &    0.021     &   21.552    &            0.0045             &          \multicolumn{2}{c}{ }           \\
		6      &     0.096     &                &    0.095     &    0.096     &    9.909    &            0.0095             &          \multicolumn{2}{c}{ }           \\
		7      &     0.598     &                &    0.595     &    0.602     &    3.476    &            0.0208             &          \multicolumn{2}{c}{ }           \\
		8      &     4.323     &                &    4.301     &    4.390     &    2.126    &            0.0919             &          \multicolumn{2}{c}{ }           \\
		9      &    33.103     &                &    32.920    &    33.462    &    1.000    &            0.3312             &          \multicolumn{2}{c}{ }           \\
		10     &    265.952    &                &   264.770    &   270.531    &    0.364    &            0.9682             &          \multicolumn{2}{c}{ }           \\ \hline\hline
		&                                              \multicolumn{8}{c}{\textbf{\textit{TRILLInc} setting (S1)}}                                              \\ \hline
		2      &     0.004     &      1.07      &    0.004     &    0.004     &    9.069    &            0.0003             &    1.187    &          0.000045          \\
		3      &     0.005     &      1.13      &    0.005     &    0.005     &   16.603    &            0.0008             &    2.057    &          0.000103          \\
		4      &     0.009     &      1.20      &    0.009     &    0.009     &   21.548    &            0.0019             &    4.124    &          0.000370          \\
		5      &     0.026     &      1.22      &    0.025     &    0.027     &   17.811    &            0.0046             &    3.015    &          0.000773          \\
		6      &     0.115     &      1.20      &    0.114     &    0.116     &    8.262    &            0.0095             &    1.860    &          0.002132          \\
		7      &     0.703     &      1.18      &    0.700     &    0.707     &    3.020    &            0.0212             &    0.865    &          0.006085          \\
		8      &     5.050     &      1.17      &    5.033     &    5.075     &    1.948    &            0.0984             &    0.529    &          0.026709          \\
		9      &    38.376     &      1.16      &    38.237    &    38.612    &    0.899    &            0.3451             &    0.050    &          0.019103          \\
		10     &    306.179    &      1.15      &   304.847    &   308.125    &    0.347    &            1.0611             &    0.106    &          0.324672          \\ \hline
		\multicolumn{2}{r|}{ } &      1.16      & \multicolumn{6}{l}{\textbf{Avg on results of size 2 to 10}}                                                          \\
		\multicolumn{2}{r|}{ } &      1.18      & \multicolumn{6}{l}{\textbf{Avg on results of size 5 to 10}}                                                          \\ \hline\hline
		&                                              \multicolumn{8}{c}{\textbf{\textit{TRILLInc} setting (S2)}}                                              \\ \hline
		2      &     0.004     &      1.21      &    0.004     &    0.005     &   10.707    &            0.0005             &    7.683    &          0.00033           \\
		3      &     0.006     &      1.26      &    0.006     &    0.006     &   16.614    &            0.0009             &    9.801    &          0.00055           \\
		4      &     0.010     &      1.34      &    0.010     &    0.010     &   19.688    &            0.0020             &   13.335    &          0.00134           \\
		5      &     0.028     &      1.34      &    0.028     &    0.028     &   16.048    &            0.0045             &   11.050    &          0.00309           \\
		6      &     0.541     &      5.67      &    0.538     &    0.545     &    1.760    &            0.0095             &   79.179    &          0.42867           \\
		7      &     2.376     &      3.97      &    2.371     &    2.383     &    0.901    &            0.0214             &   70.683    &          1.67948           \\
		8      &    25.473     &      5.89      &    25.327    &    25.761    &    0.389    &            0.0991             &   80.171    &          20.42165          \\
		9      &    368.582    &     11.13      &   365.968    &   374.780    &    0.094    &            0.3460             &   89.586    &         330.19911          \\
		10     &    736.830    &      2.77      &   729.724    &   744.526    &    0.144    &            1.0627             &   58.399    &         430.29776          \\ \hline
		\multicolumn{2}{r|}{ } &      3.84      & \multicolumn{6}{l}{\textbf{Avg on results of size 2 to 10}}                                                          \\
		\multicolumn{2}{r|}{ } &      5.13      & \multicolumn{6}{l}{\textbf{Avg on results of size 5 to 10}}                                                          \\ \hline\hline
		&                                              \multicolumn{8}{c}{\textbf{\textit{TRILLInc} setting (S3)}}                                              \\ \hline
		2      &     0.005     &      1.44      &    0.005     &    0.006     &   12.722    &            0.0006             &   16.859    &          0.00085           \\
		3      &     0.048     &     10.67      &    0.047     &    0.048     &    3.429    &            0.0016             &   86.679    &          0.04119           \\
		4      &     5.321     &     711.84     &    5.307     &    5.334     &    0.069    &            0.0037             &   99.737    &          5.30659           \\
		5      &    830.989    &    39655.17    &   812.261    &   850.088    &    0.001    &            0.0084             &   99.994    &         830.93938          \\
		6      &  39,779.910   &   416509.56    &  39133.263   &  40515.894   &    0.000    &            0.0192             &   99.999    &        39779.65049         \\
		7      &      --       &       --       &      --      &      --      &     --      &              --               &     --      &             --             \\ \hline
		\multicolumn{2}{r|}{ } &    91377.73    & \multicolumn{6}{l}{\textbf{Avg on results of size 2 to 6}}                                                           \\
		\multicolumn{2}{r|}{ } &   228082.36    & \multicolumn{6}{l}{\textbf{Avg on results of size 5 to 6}}                                                           \\ \hline\hline
		&                                              \multicolumn{8}{c}{\textbf{\textit{TRILLInc} setting (S4)}}                                              \\ \hline
		2      &     0.004     &      1.25      &    0.004     &    0.004     &   11.543    &            0.0005             &    1.896    &          0.000083          \\
		3      &     0.006     &      1.36      &    0.006     &    0.006     &   16.109    &            0.0010             &    2.885    &          0.000175          \\
		4      &     0.012     &      1.66      &    0.012     &    0.013     &   16.888    &            0.0021             &    4.410    &          0.000549          \\
		5      &     0.043     &      2.04      &    0.042     &    0.045     &   10.408    &            0.0045             &    2.270    &          0.000971          \\
		6      &     0.226     &      2.36      &    0.225     &    0.227     &    4.458    &            0.0101             &    0.980    &          0.002212          \\
		7      &     1.512     &      2.53      &    1.505     &    1.526     &    1.481    &            0.0224             &    1.730    &          0.026159          \\
		8      &    10.933     &      2.53      &    10.891    &    11.077    &    0.770    &            0.0842             &    0.439    &          0.048044          \\
		9      &    84.977     &      2.57      &    84.625    &    85.346    &    0.420    &            0.3573             &    0.032    &          0.026987          \\
		10     &    676.704    &      2.54      &   674.033    &   680.020    &    0.077    &            0.5182             &    0.094    &          0.637594          \\ \hline
		\multicolumn{2}{r|}{ } &      2.09      & \multicolumn{6}{l}{\textbf{Avg on results of size 2 to 10}}                                                          \\
		\multicolumn{2}{r|}{ } &      2.43      & \multicolumn{6}{l}{\textbf{Avg on results of size 5 to 10}}                                                          \\ \hline
	\end{tabular}
\end{table}

\begin{table}
	\centering
	\caption{Average running time in seconds computed on 10 executions of the query with the original version of the reasoner BUNDLE and the extended version of the reasoner \textit{BUNDLEInc} in settings \textbf{(S1)}, \textbf{(S2)}, \textbf{(S3)}, and \textbf{(S4)} for scenario \textbf{(SC1.2)}, together with the minimum and maximum running time, and ratio of the two average running time measurements, the times for computing the DISPONTE probability, and for repair semantics computation in seconds with the percentage of the total running time. The average of the ratios considering all and only sizes greater than 4 are also shown. The value ``--'' means that system crashed due to out of memory error raised by the library for the management of the BDD when computing the repair semantics.\label{tab:test-ext-res-b-sc2}}	
	\tiny
	\begin{tabular}{c|c|c|c|c|c|c|c|c}
		& \textbf{Time} &                & \textbf{Min} & \textbf{Max} & \multicolumn{2}{c|}{\textbf{DISPONTE Time}} & \multicolumn{2}{c}{\textbf{Repair Time}} \\
		\textbf{i} & \textbf{(s)}  & \textbf{Ratio} & \textbf{(s)} & \textbf{(s)} & \textbf{\%} &         \textbf{(s)}          & \textbf{\%} &        \textbf{(s)}        \\ \hline\hline
		&                                                          \multicolumn{8}{c}{\textbf{BUNDLE}}                                                          \\ \hline
		2      &    0.8798     &                &    0.829     &    0.943     &    5.160    &             0.045             &          \multicolumn{2}{c}{ }           \\
		3      &    1.3778     &                &    1.325     &    1.412     &    3.985    &             0.055             &          \multicolumn{2}{c}{ }           \\
		4      &    2.0821     &                &    1.919     &    2.298     &    2.718    &             0.057             &          \multicolumn{2}{c}{ }           \\
		5      &    3.4229     &                &    3.010     &    3.799     &    1.794    &             0.061             &          \multicolumn{2}{c}{ }           \\
		6      &    5.5915     &                &    5.011     &    6.182     &    1.295    &             0.072             &          \multicolumn{2}{c}{ }           \\
		7      &    9.3974     &                &    8.315     &    9.965     &    0.988    &             0.093             &          \multicolumn{2}{c}{ }           \\
		8      &    16.0498    &                &    14.791    &    17.402    &    0.647    &             0.104             &          \multicolumn{2}{c}{ }           \\
		9      &    27.5138    &                &    26.572    &    29.035    &    0.652    &             0.180             &          \multicolumn{2}{c}{ }           \\
		10     &    49.4283    &                &    47.489    &    51.253    &    0.585    &             0.289             &          \multicolumn{2}{c}{ }           \\ \hline\hline
		&                                             \multicolumn{8}{c}{\textbf{\textit{BUNDLEInc} setting (S1)}}                                              \\ \hline
		2      &     1.596     &      1.81      &    1.423     &    1.781     &    2.782    &             0.044             &    0.081    &           0.0013           \\
		3      &    3.0539     &      2.22      &    2.537     &    3.652     &    1.624    &             0.050             &    0.039    &           0.0012           \\
		4      &    5.4023     &      2.60      &    4.687     &    5.822     &    1.050    &             0.057             &    0.030    &           0.0016           \\
		5      &    10.0865    &      2.95      &    9.455     &    10.794    &    0.576    &             0.058             &    0.026    &           0.0026           \\
		6      &    18.4722    &      3.30      &    17.026    &    19.73     &    0.420    &             0.078             &    0.022    &           0.0041           \\
		7      &    35.2566    &      3.75      &    34.262    &    36.386    &    0.260    &             0.092             &    0.020    &           0.0072           \\
		8      &    69.5163    &      4.33      &    68.145    &    70.735    &    0.150    &             0.104             &    0.016    &           0.011            \\
		9      &   147.8663    &      5.37      &   146.227    &    150.82    &    0.131    &             0.194             &    0.014    &           0.0205           \\
		10     &   313.2454    &      6.34      &   305.368    &   319.485    &    0.120    &             0.376             &    0.014    &           0.0447           \\ \hline
		\multicolumn{2}{r|}{ } &      3.63      & \multicolumn{6}{l}{\textbf{Avg on results of size 2 to 10}}                                                          \\
		\multicolumn{2}{r|}{ } &      4.34      & \multicolumn{6}{l}{\textbf{Avg on results of size 5 to 10}}                                                          \\ \hline\hline
		&                                             \multicolumn{8}{c}{\textbf{\textit{BUNDLEInc} setting (S2)}}                                              \\ \hline
		2      &    2.0433     &      2.32      &    1.792     &    2.280     &    2.227    &             0.046             &    0.347    &           0.0071           \\
		3      &    3.1153     &      2.26      &    2.786     &    3.754     &    1.483    &             0.046             &    0.331    &           0.0103           \\
		4      &    5.8579     &      2.81      &    4.938     &    6.326     &    0.925    &             0.054             &    0.210    &           0.0123           \\
		5      &    10.3787    &      3.03      &    9.258     &    11.143    &    0.577    &             0.060             &    0.195    &           0.0202           \\
		6      &    18.5442    &      3.32      &    17.814    &    19.914    &    0.438    &             0.081             &    0.220    &           0.0408           \\
		7      &    35.9427    &      3.82      &    35.071    &    37.207    &    0.260    &             0.093             &    0.215    &           0.0771           \\
		8      &    71.4012    &      4.45      &    70.196    &    72.708    &    0.160    &             0.114             &    0.251    &           0.1792           \\
		9      &   148.5623    &      5.40      &    145.67    &   152.392    &    0.118    &             0.175             &    0.235    &           0.3492           \\
		10     &   314.0794    &      6.35      &   307.931    &   325.788    &    0.104    &             0.328             &    0.274    &           0.8612           \\ \hline
		\multicolumn{2}{r|}{ } &      3.75      & \multicolumn{6}{l}{\textbf{Avg on results of size 2 to 10}}                                                          \\
		\multicolumn{2}{r|}{ } &      4.40      & \multicolumn{6}{l}{\textbf{Avg on results of size 5 to 10}}                                                          \\ \hline\hline
		&                                             \multicolumn{8}{c}{\textbf{\textit{BUNDLEInc} setting (S3)}}                                              \\ \hline
		2      &    2.4621     &      2.80      &     2.15     &     2.73     &    1.974    &             0.049             &    0.585    &           0.0144           \\
		3      &    4.7394     &      3.44      &    4.127     &    5.439     &    1.234    &             0.059             &    1.015    &           0.0481           \\
		4      &    9.4466     &      4.54      &    8.588     &    10.707    &    0.620    &             0.059             &    5.452    &           0.515            \\
		5      &    52.9568    &     15.47      &    50.687    &    55.01     &    0.137    &             0.073             &   69.385    &          36.7439           \\
		6      &      --       &       --       &      --      &      --      &     --      &              --               &     --      &             --             \\ \hline
		\multicolumn{2}{r|}{ } &      6.56      & \multicolumn{6}{l}{\textbf{Avg on results of size 2 to 10}}                                                          \\
		\multicolumn{2}{r|}{ } &     15.47      & \multicolumn{6}{l}{\textbf{Avg on results of size 5 to 10}}                                                          \\ \hline\hline
		&                                             \multicolumn{8}{c}{\textbf{\textit{BUNDLEInc} setting (S4)}}                                              \\ \hline
		2      &    2.7812     &      3.16      &    2.385     &    3.089     &    1.664    &             0.046             &    0.137    &           0.0038           \\
		3      &    4.2761     &      3.10      &    3.565     &    4.744     &    1.202    &             0.051             &    0.119    &           0.0051           \\
		4      &    6.7704     &      3.25      &    5.953     &    7.452     &    0.857    &             0.058             &    0.093    &           0.0063           \\
		5      &    11.996     &      3.50      &    11.552    &    12.595    &    0.561    &             0.067             &    0.082    &           0.0098           \\
		6      &    21.9944    &      3.93      &    21.111    &    22.938    &    0.325    &             0.071             &    0.081    &           0.0179           \\
		7      &    41.0034    &      4.36      &    40.017    &    42.005    &    0.229    &             0.094             &    0.061    &           0.0252           \\
		8      &    81.4007    &      5.07      &    80.235    &    83.102    &    0.156    &             0.127             &    0.063    &           0.0511           \\
		9      &   173.5887    &      6.31      &    169.63    &   182.597    &    0.122    &             0.211             &    0.069    &           0.1206           \\
		10     &    360.365    &      7.29      &   354.149    &   370.689    &    0.091    &             0.326             &    0.059    &           0.2144           \\ \hline
		\multicolumn{2}{r|}{ } &      4.44      & \multicolumn{6}{l}{\textbf{Avg on results of size 2 to 10}}                                                          \\
		\multicolumn{2}{r|}{ } &      5.08      & \multicolumn{6}{l}{\textbf{Avg on results of size 5 to 10}}                                                          \\ \hline
	\end{tabular}
\end{table}

In this scenario, the ratio is larger than in \textbf{(SC1)}, due to the larger number of probabilistic axioms producing larger BDDs for checking the repair semantics. In fact, in this case, the time for computing the probability of the query is negligible and its relative contribution to the running time decreases with the increase of $i$, similarly to scenario \textbf{(SC1.1)}. The DISPONTE time is always below 1 second except for $i=10$ in settings \textbf{(S1)} and \textbf{(S2)}. On the other hand, the repair time in this scenario is a large portion of the whole running time when the number of justifications for the query is high in settings \textbf{(S2)} and \textbf{(S3)}, while it shows a behaviour in setting \textbf{(S1)} similar to that in \textbf{(S4)}, where the KB is consistent and the query has only two justifications respectively.

\paragraph{Scenario (SC1.3).} Tables~\ref{tab:test-ext-res-t-sc3} and~\ref{tab:test-ext-res-b-sc3} show the running time for \textit{TRILLInc} and \textit{BUNDLEInc} when we increase the number of probabilistic axioms in the KB with $i=5$ of setting \textbf{(S3)}. The columns of the tables include the running time in seconds, DISPONTE and Repair time both in seconds and as a fraction of the total running time.

\begin{table}
	\centering
	\caption{Average running time in seconds computed on 10 executions of the query with the reasoner \textit{TRILLInc} for scenario \textbf{(SC1.3)} and $i=5$, together with the minimum and maximum running time, the times for computing the DISPONTE probability, and for repair semantics computation in seconds with the percentage of the total running time.\label{tab:test-ext-res-t-sc3}}	
	\tiny
	\begin{tabular}{c|c|c|c|c|c|c|c|c}
		& \textbf{Time} & \textbf{Min} & \textbf{Max} & \multicolumn{2}{c|}{\textbf{DISPONTE Time}} & \multicolumn{2}{c}{\textbf{Repair Time}} & \textbf{Repair} \\
		\textbf{\% Prob. axioms}           & \textbf{(s)}  & \textbf{(s)} & \textbf{(s)} & \textbf{\%} &         \textbf{(s)}          & \textbf{\%} &        \textbf{(s)}        & \textbf{Answer} \\ \hline\hline
		Certain KB (0\%)               &    0.0451     &    0.0448    &    0.0457    &   0.0000    &            0.0000             &   0.9983    &           0.0005           &      False      \\
		+ assertion (6\%)              &    0.0455     &    0.0452    &    0.0459    &   9.4396    &            0.0043             &   1.0544    &           0.0005           &      False      \\
		+ disjoint ax. (12\%)            &    0.0544     &    0.0541    &    0.0547    &   8.1064    &            0.0044             &   6.1312    &           0.0033           &      Brave      \\
		+ $B_{i-1}\sqsubseteq P_i\sqcap Q_i$ (41\%) &    0.0677     &    0.0674    &    0.0686    &   7.2031    &            0.0049             &   11.2960   &           0.0076           &      Brave      \\
		+ $P_{i}\sqsubseteq B_i$ (71\%)       &    55.6765    &   55.5751    &   55.8728    &   0.0095    &            0.0053             &   19.9938   &          11.1319           &      Brave      \\
		+ $Q_{i}\sqsubseteq B_i$ (100\%)       &   830.9887    &   812.2609   &   850.0875   &   0.0010    &            0.0084             &   99.9941   &          830.9394          &      Brave      \\ \hline
	\end{tabular}
\end{table}

\begin{table}
	\centering
	\caption{Average running time in seconds computed on 10 executions of the query with the reasoner \textit{BUNDLEInc} for scenario \textbf{(SC1.3)} and $i=5$, together with the minimum and maximum running time, the times for computing the DISPONTE probability, and for repair semantics computation in seconds with the percentage of the total running time.\label{tab:test-ext-res-b-sc3}}	
	\tiny
	\begin{tabular}{c|c|c|c|c|c|c|c|c}
		& \textbf{Time} & \textbf{Min} & \textbf{Max} & \multicolumn{2}{c|}{\textbf{DISPONTE Time}} & \multicolumn{2}{c}{\textbf{Repair Time}} & \textbf{Repair} \\
		\textbf{\% Prob. axioms}           & \textbf{(s)}  & \textbf{(s)} & \textbf{(s)} & \textbf{\%} &         \textbf{(s)}          & \textbf{\%} &        \textbf{(s)}        & \textbf{Answer} \\ \hline\hline
		Certain KB (0\%)               &    16.5428    &    15.939    &    17.54     &   0.0000    &            0.0000             &   0.0375    &           0.0062           &      False      \\
		+ assertion (6\%)              &    17.1649    &   15.9390    &   18.7200    &   0.1573    &            0.0270             &   0.0268    &           0.0046           &      False      \\
		+ disjoint ax. (12\%)            &    17.0809    &   15.8860    &   18.7200    &   0.2061    &            0.0352             &   0.0285    &           0.0049           &      Brave      \\
		+ $B_{i-1}\sqsubseteq P_i\sqcap Q_i$ (41\%) &    17.1833    &   15.8860    &   18.7200    &   0.2307    &            0.0397             &   0.0300    &           0.0052           &      Brave      \\
		+ $P_{i}\sqsubseteq B_i$ (71\%)       &    17.2612    &   15.8860    &   18.7200    &   0.2644    &            0.0456             &   0.7853    &           0.1356           &      Brave      \\
		+ $Q_{i}\sqsubseteq B_i$ (100\%)       &    52.9568    &   50.6870    &   55.0100    &   0.1371    &            0.0726             &   69.3847   &          36.7439           &      Brave      \\ \hline
	\end{tabular}
\end{table}

From the results of this scenario, we can see that, with the increase of the percentage of probabilistic axioms, the complexity of computing the repair semantics becomes larger. Moreover, when all the axioms are probabilistic, the Repair time is large. The main reason is that the system builds BDDs associating Boolean variables only to probabilistic axioms. When all the axioms are probabilistic, each justification is associated with a different set of Boolean variables. Otherwise, if there are certain axioms, there would be different justifications associated with the same set of variables. For example, consider the set of justifications $\{\{ax_1,ax_2,ax_3\},\{ax_1,ax_2,ax_4\}\}$, if all the axioms are probabilistic, then the Boolean formula represented by the BDD is $(X_{ax_1}\vee X_{ax_2}\vee X_{ax_3})\wedge(X_{ax_1}\vee X_{ax_2}\vee X_{ax_4})$, while if $ax_3$ and $ax_4$ are certain, the Boolean formula is $(X_{ax_1}\vee X_{ax_2})\wedge(X_{ax_1}\vee X_{ax_2})=X_{ax_1}\vee X_{ax_2}$, with $X_k$ the Boolean variable associated to axiom $k$. Therefore, having a KB also containing certain axioms significantly reduces the size of the BDD representing the set of justifications and the number of operations to build it. Similarly, it also reduces the BDD combination operations required to compute the repair semantics and the number of Boolean variables $X_{\cC,\cB}$ required for the check of the AR semantics. These results confirm the fact that the increase of probabilistic axioms also increases the complexity of the reasoning because there are more possibility to create consistent worlds and repairs.

From the results we can also see that \textit{TRILLInc} is faster when the percentage of probabilistic axioms is less than 50\%, and its Repair time tends to increase faster than that of \textit{BUNDLEInc}. However, this strictly depends on the BDD library, which unfortunately changes for the two systems.

\paragraph{Discussions.} Overall, the computation of the DISPONTE probability is the task with the lower impact on the whole running time except when the problem is simple, for example, due to a small KB. However, in these cases the execution takes less than one second. The DISPONTE time is slightly affected by the extensions. On the other hand, the overhead added by the repair semantics heavily depends on the query and the number of probabilistic axioms considered. It is important to highlight that the performance is strictly tied to the tableau implementation, which is an approach that requires significant time and resources and could be the bottleneck of the whole computation in many cases. If we are not interested in computing the repair semantics, the systems are significantly faster: in the worst case reported in the results, the running time drops from more than 11 hours to less than 0.3 seconds.

Thus, the answer to \textbf{Q1} is that the proposed extension can hugely impact the running time, with the Repair time dominating the running time. It is important to note, however, that the proposed extension is not specifically tailored to computing repairs, for which there are dedicated systems, such as CQApri (considered for the next question). At the moment, both reasoners allow the user to choose whether to compute it together with the DISPONTE probability. Finally, it is important to highlight that the time ratios between the original reasoners and their extended versions are always affected by the fact that there are more justifications to find when the KB is inconsistent.

\subsection{Q2. How does our approach compare with CQApri?}\label{subsec:q2}
For the second question, we considered a KB used by the authors of CQApri\footnote{Available at \url{https://lahdak.lri.fr/CQAPri}}~\cite{DBLP:journals/jair/BienvenuBG19}, a system for querying DL-Lite KB under AR, IAR and Brave semantics implemented in Java, which exploits a relational database to store the assertions of the KB, a query rewriting engine for the DL-Lite language and a SAT solver to answer queries under AR semantics. The authors of CQApri used a simplified version of the Lehigh University Benchmark (LUBM) ontology~\cite{DBLP:journals/ws/GuoPH05}, which differs from the original version for the removal of the axioms that cannot be modelled by DL-Lite. For this version, they generated different ABoxes of increasing size containing inconsistencies w.r.t. the simplified LUBM TBox.

For the comparison, we ran \textit{BUNDLEInc} against CQApri. Only \textit{BUNDLEInc} is considered here because \textit{TRILLInc} took more than 12 hours to answer queries with the KBs considered in this test. Several optimizations have been developed to make the original version of TRILL more efficient~\cite{DBLP:journals/ws/ZeseC21}, however, some of them cannot be used in \textit{TRILLInc} because they were implemented considering consistent KBs. As already discussed, similar problems can be found in \textit{BUNDLEInc}, which can only use the ``OWLExplanation'' justification generator.

First, we considered version \verb|u1conf0|, a consistent KB, from which we build an OWL KB to use with \textit{BUNDLEInc} containing all the assertions that CQApri stores in the database. The resulting KB models one university and contains 108,864 axioms in the ABox (28,023 class assertions, 47,640 object property assertions and 33,201 data property assertions) for a total of 127,320 axioms in the KB. We made all assertional axioms probabilistic in order to consider the same repairs with both systems.

From this KB, we randomly created 200 Boolean queries of the form $a_i:C_i$ and $(a_i,b_i):R_i$ by sampling individuals $a_i$ and $b_i$, classes $C_i$ and object properties $R_i$. Each of these 200 queries is built to have at least one justification.  We have also generated 200 queries in the same way but having zero justifications, in order to test the performance of \textit{BUNDLEInc} even when the query is not entailed. 
Each query has been run 10 times to ensure that the test machine's load did not affect the results. We computed the relative standard deviation of the 10 executions for each query separately, and it varied from 0.04\% to 6.04\%. The average relative standard deviation, calculated on the 200 values computed (one for each query) was 1.56\% with standard deviation of 1.32\%, first quartile 0.54\%, median 1.19\%, third quartile 2.19\%. This means that the running time of a query executed more times is expected to vary between 0.24\% and 2.88\%. In this way we ensured that the results obtained are not affected by the machine's load and we can concentrate on the results of the test for \textbf{Q2}.

Table~\ref{tab:bundle-cqapri} shows the results in terms of averaged running time in seconds for \textit{BUNDLEInc} and CQApri to solve the 200 queries with justifications. The table shows the average running time with standard deviation across all the queries, the minimum and maximum running times, median, quartiles 1 and 3, with the inter quartile range (IQR) and the number of outliers (computed by counting the number of queries outside the range $[Q1-1.5\cdot QIR,Q3+1.5\cdot QIR]$). The queries considered as outliers for BUNDLE and for CQApri are the same except for 2 of them. In these 2 cases, the running time was outside but close to the range. The correspondence between the outliers for \textit{BUNDLEInc} and CQApri indicates that, among the 200 queries, there are 15.5\% of them that are significantly more complex to solve, the main difference between the two systems is that \textit{BUNDLEInc}'s running time is more dependent on how hard it is to find the justifications for the tableau algorithm, as it can be seen in Figure~\ref{fig:boxplot}.
Table~\ref{tab:bundle} contains more details about the running time of \textit{BUNDLEInc} when solving the 200 queries without justifications and the 200 queries with justifications. For the latter case, the 200 queries are divided in queries with 1 justification and queries with more than 1 justifications. CQApri took the same time to answer queries both with and without justifications, so we do no report its results in Table~\ref{tab:bundle} because they coincide with those reported in Table~\ref{tab:bundle-cqapri}.

\begin{table}
	\caption{Running time in seconds taken by \textit{BUNDLEInc} and CQApri to answer the set of 200 queries with at least one justification. The table shows the average running time with  standard deviation, and the minimum and maximum running time for both systems.}\label{tab:bundle-cqapri}
	\centering
	\footnotesize
	\begin{tabular}{c|c|c|c|c|c|c|c|c}
		\textbf{System}   &   \textbf{Avg. running time}   &                         \multicolumn{6}{c|}{\textbf{Running time (s)}}  &\textbf{N.}                        \\
		& \textbf{$\pm$ std. dev. (s)} & \textbf{Min} & \textbf{Q1} & \textbf{Median} & \textbf{Q3} & \textbf{Max} & \textbf{IQR}  & \textbf{Outliers} \\ \hline\hline
		CQApri       &         11.80 $\pm$ 0.08         &    11.74     &    11.76    &      11.77      &    11.80    &    12.04     & 0.04 &     33      \\
		\textit{BUNDLEInc} &       714.55 $\pm$ 1994.14        &     4.50     &    17.08     &     36.72       &    44.17     &   11,832.96   & 27.09 &     31      \\ \hline
	\end{tabular}
\end{table}

\begin{table}
	\caption{Statistics about \textit{BUNDLEInc} to answer queries with at least one justification and queries without justifications. The table shows the average running time in seconds with  standard deviation, the minimum and maximum running time in the cases the query has 0, 1 and more than 1 justifications. Moreover, in the case of more than 1 justifications, the table shows the average number with  standard deviation, the minimum and maximum number of justifications.}\label{tab:bundle}
	\centering
	\footnotesize
	\begin{tabular}{c|c|c}
		\hline
		\multicolumn{3}{c}{\textbf{0 justifications}}                               \\ \hline
		Avg. running time (s) $\pm$ std. dev. (s) &     Min. running time (s)     &     Max. running time (s)     \\ \hline
		4.685 $\pm$ 0.11              &             4.369             &             5.057             \\ \hline\hline
		\multicolumn{3}{c}{\textbf{1 justification}}                                \\ \hline
		Avg. running time (s) $\pm$ std. dev. (s) &     Min. running time (s)     &     Max. running time (s)     \\ \hline
		28.420 $\pm$ 13.02             &             4.504             &            51.230             \\ \hline\hline
		\multicolumn{3}{c}{\textbf{More than 1 justifications}}                          \\ \hline
		Avg. running time (s) $\pm$ std. dev. (s) &     Min. running time (s)     &     Max. running time (s)     \\ \hline
		4,383.834 $\pm$ 3,057.68          &            480.206            &          11,832.960           \\ \hline
		Avg. number of just.  $\pm$ std. dev.   & Min. number of justifications & Max. number of justifications \\ \hline
		35.32 $\pm$ 19.29             &               9               &              79               \\ \hline
	\end{tabular}
\end{table}

As it can be seen from Table \ref{tab:bundle-cqapri}, \textit{BUNDLEInc}  has an average running time of less than 715 seconds, while CQApri can solve the same queries in less than 12 seconds on average. 
The computation time for the repair semantics with \textit{BUNDLEInc} was 0 because the KB is consistent.  The average DISPONTE time on the set of 200 queries with justifications is $0.09\ (\pm 0.49)$ milliseconds and $0.05\ (\pm 0.22)$ milliseconds for queries with 0 justifications, while the time for computing the repair semantics is $0.07\ (\pm 0.50)$ milliseconds considering only queries with 1 justification, and is $0.13\ (\pm 0.42)$ milliseconds in case of more than 1 justifications. Maybe this small value is due to the number of probabilistic axioms contained in the justifications, which varies between 1 and 3, making the management of the BDDs significantly simpler than in the experiments for \textbf{Q1}. 

The reason for the high standard deviation value for \textit{BUNDLEInc} in Table~\ref{tab:bundle-cqapri} lies probably in the variability of the number of justifications. In particular, there are 10 queries with a number of justifications between 9 and 20 and an average time of 859 seconds, and 21 queries with more than 20 justifications that required an average of 6,268 seconds, with two queries having 71 and 79 justifications and requiring a running time of 9,408.328 and 11,832.960 seconds respectively to be answered. These 31 queries are the 31 outliers reported in Table~\ref{tab:bundle-cqapri}, i.e., queries with running time higher than the range computed on the IQR. If we remove these outliers,  \textit{BUNDLEInc}  is generally slower than CQApri but with comparable times (same order of magnitude,  with an average running time of less than 48 seconds instead of more than the 714 seconds reported in Table \ref{tab:bundle-cqapri}), as shown in Figure~\ref{fig:boxplot}, still with a significant variability in the running times.

\begin{figure}[t]
	\centering
	\includegraphics[width=\linewidth]{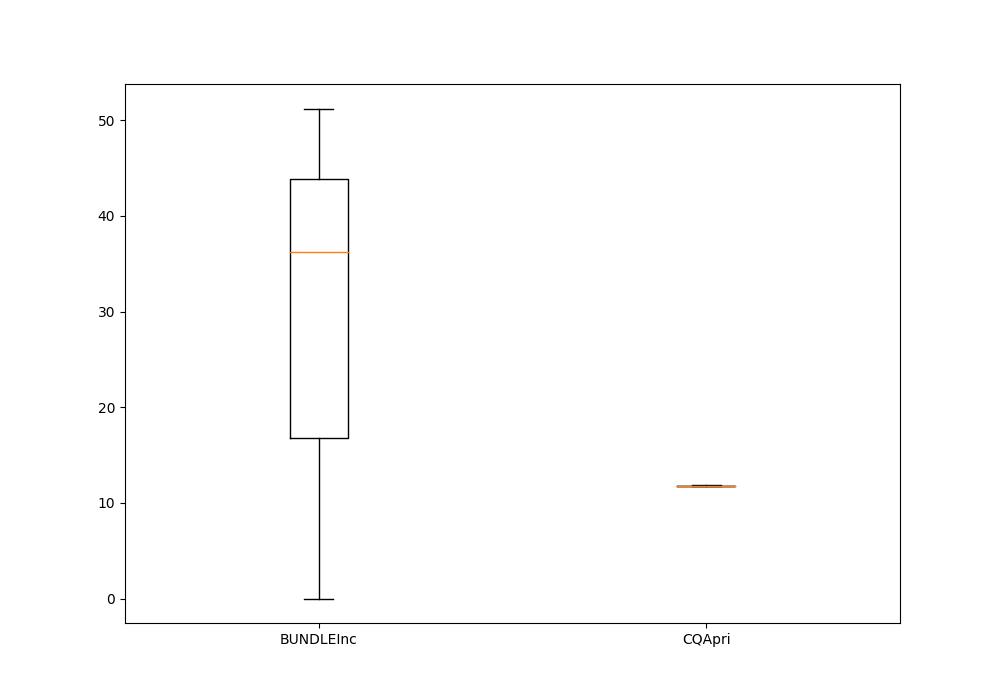}
	\caption{Boxplot with outliers removed of the running time in seconds taken by \textit{BUNDLEInc} and CQApri to answer the set of 200 queries with at least one justification.\label{fig:boxplot}}
\end{figure}

Then we considered  version \verb|u1conf1| of  CQApri's benchmark. This version differs from the previous because the authors added 61 assertion axioms making the KB inconsistent. The resulting KB contains 108,925 axioms and the inconsistency has more than 40 justifications. In this case, while the performance of CQApri was similar to that on \verb|u1conf0|, the search for the justifications with \textit{BUNDLEInc} takes a time more than the 12 hours timeout. This result led us to the third question, analysed in Section~\ref{subsec:q3}.

\paragraph{Discussions.} From the results it appears that \textit{BUNDLEInc} is more affected by the complexity of the query, due to the tableau algorithm, than CQApri, which presents a running time that is more stable and more dependent on the size of the KB than on the complexity of the query. Overall, CQApri is able to cope with larger KBs in a time which is  several orders of magnitude smaller than \textit{BUNDLEInc}, especially when we add inconsistency to the KB. 
However, it is important to bear in mind the main differences between the two systems:
\begin{itemize}
	\item CQApri is tailored to DL-Lite, so it imposes a stronger limitation on the expressivity than \textit{BUNDLEInc}, which can deal with OWL 2 (\sroiqd). This means that CQApri cannot be used to ask the queries of Example~\ref{ex:univ-empl} because complex concepts such as intersections can be used only as superclasses.
	\item CQApri answers queries under the Brave, AR, and IAR semantics, while \textit{BUNDLEInc}, besides these semantics, computes the probability of queries under DISPONTE and returns the set of justifications for both the inconsistency and the query, allowing a full analysis of the query and the debug of the KB at the same time. CQApri can return the repairs, but in this experiment we did not ask for those, because we wanted to concentrate on query answering. We would like to point out that, also for queries with 79 justifications, the Repair time in \textit{BUNDLEInc} is almost instantaneous (average over all the 200 queries of $0.0001 \pm 0.0003$ seconds). 
	\item CQApri can also answer conjunctive queries, while \textit{BUNDLEInc} answers Boolean queries only.
\end{itemize}

\subsection{Q3. How does our approach scale when considering complex KBs?}\label{subsec:q3}
To answer the third question, we analysed the trend of the running time in three scenarios:
\begin{description}
	\item[(SC3.1)] adding each of the 61 assertions causing inconsistency in \verb|u1conf1| one by one;
	\item[(SC3.2)] varying the number of probabilistic axioms in \verb|u1conf0|, with or without the addition of one assertion taken from the 61 causing inconsistency in \verb|u1conf1|;
	\item[(SC3.3)] considering the whole \verb|u1conf1| KB and imposing a limit on the number of justifications to find. Collecting a limited number of justifications does not allow the correct computation of the repair semantics and of the correct value of the DISPONTE probability, which will be a lower bound of the real probability of the query.
\end{description}

\paragraph{Scenario (SC3.1).} For this test we created different KBs starting from \verb|u1conf0| and adding an increasing number of assertions taken from the set of 61 causing inconsistency in \verb|u1conf1|. We slightly modified the KB in order to obtain a query true under AR semantics and we generated two queries: $Q_{IAR}$  always true under IAR semantics and $Q_{AR}$ true under AR semantics when the KB becomes inconsistent. Both queries have three justifications, one composed of two axioms and two of four axioms. We collected the average total, DISPONTE and Repair time, and the number of justifications for inconsistency. Table \ref{tab:bundle-inc-up} shows the results obtained in this scenario: the assertions add many conflicts, and so many justifications for inconsistency. We can see that DISPONTE and Repair time is small, while the time to find the justifications rapidly increases becoming larger than 12 hours with four assertions. As expected, query answering under the AR semantics is more expensive than under the IAR semantics. The time to compute the truth value under AR semantics is similar to that of the Brave and IAR semantics.  
\begin{table}
	\centering
	\caption{Average running time in seconds needed by  \textit{BUNDLEInc} for computing the DISPONTE probability and the repair semantics for the IAR and the AR queries. Values ``--'' mean timeout (set to 12 hours).\label{tab:bundle-inc-up}}	
	\tiny
	\begin{tabular}{c|c|c|c|c|c|c|c|c}
		\textbf{\# added}  & \textbf{Query} & \textbf{Repair} & \textbf{Time} & \multicolumn{2}{c|}{\textbf{DISPONTE Time}} & \multicolumn{2}{c|}{\textbf{Repair Time}} &   \textbf{\# Incons.}   \\
		\textbf{assertions} & \textbf{Sem.}  &                 & \textbf{(s)}  & \textbf{\%} &         \textbf{(s)}          & \textbf{\%} &        \textbf{(s)}         & \textbf{Justifications} \\ \hline\hline
		\multirow{2}{*}{0}  &   $Q_{IAR}$    &       IAR       &    178.19     &   0.0006    &             0.001             &   0.0011    &            0.002            &   \multirow{2}{*}{0}    \\
		&    $Q_{AR}$    &       IAR       &    111.708    &   0.0009    &             0.001             &   0.0027    &            0.003            &                         \\ \hline
		\multirow{2}{*}{1}  &   $Q_{IAR}$    &       IAR       &    260.790    &   0.0004    &             0.001             &   0.0008    &            0.002            &   \multirow{2}{*}{2}    \\
		&    $Q_{AR}$    &       AR        &    392.499    &   0.0003    &             0.001             &   0.0013    &            0.005            &                         \\ \hline
		\multirow{2}{*}{2}  &   $Q_{IAR}$    &       IAR       &   3,776.196   &   0.0001    &             0.003             &   0.0001    &            0.004            &   \multirow{2}{*}{27}   \\
		&    $Q_{AR}$    &       AR        &   3,401.217   &   0.0001    &             0.003             &   0.0002    &            0.007            &                         \\ \hline
		\multirow{2}{*}{3}  &   $Q_{IAR}$    &       IAR       &  18,831.849   &   0.00004   &             0.007             &   0.00002   &            0.004            &   \multirow{2}{*}{34}   \\
		&    $Q_{AR}$    &       AR        &  17,456.931   &   0.00003   &             0.006             &   0.00005   &            0.008            &                         \\ \hline
		4          &       --       &       --        &      --       &     --      &              --               &     --      &                             &                         \\ \hline
	\end{tabular}
\end{table}

\paragraph{Scenario (SC3.2).} We considered six different versions of the \verb|u1conf0| KB: \verb|u1conf0| without probabilistic axioms, with only assertions made probabilistic and with all the axioms made probabilistic, similarly to what done in scenario \textbf{(SC1.3)} of \textbf{Q1}, and \verb|u1conf0| with one assertion causing inconsistency (as the second row of Table \ref{tab:bundle-inc-up}) without probabilistic axioms, with only assertions made probabilistic and with all axioms made probabilistic. We ran the same 200 queries with at least one justification built for the experiments in \textbf{Q2}\footnote{As in \textbf{Q2}, we ran the 200 queries 10 times to be sure that the running times were not affected by variation of the machine's load. Similar to \textbf{Q2},  running time on the 10 executions for each query presented a standard deviation that varies between 0.07\% to 6.89\% (average of 1.72\%$\pm$1.42\%, first quart. 0.60\%, median 1.24\%, third quart. 2.31\%) of the average running time of the query, ensuring correctness of Table \ref{tab:bundle-prob-up}.}. From Table \ref{tab:bundle-prob-up} we can see that, for the same KB \verb|u1conf0| or inconsistent \verb|u1conf0|, the time for the computation of the DISPONTE probability is negligible independently from the number of probabilistic axioms in the KB, and it clearly increases with the number of probabilistic axioms. On the other hand, the search for justifications (that represents the largest contribution to the total running time) is weakly dependent from the number of probabilistic axioms, as expected the first three rows show similar running time, the same for the last three rows. However, interestingly, the ``OWLExplanation'' generator used by the reasoners underlying \textit{BUNDLEInc} explores the justification search space differently depending on the number of annotated axioms present in the KB, making reasoning on the KB containing all probabilistic axioms slightly faster than reasoning on the KB with only probabilistic assertions. This is probably due to a different order in which the axioms are considered when searching for justifications, an order that guides the exploration of the search space.

\begin{table}
	\centering
	\caption{Average running time in seconds needed by \textit{BUNDLEInc} for answering queries w.r.t. KB \texttt{u1conf0} and KB \texttt{u1conf0} having an axiom causing inconsistency with increasing number of probabilistic axioms.\label{tab:bundle-prob-up}}	
	\tiny
	\setlength{\tabcolsep}{3pt} 
	\renewcommand{\arraystretch}{1.1}
	\begin{tabular}{c|c|c|c|c|c|c|c}
		& \textbf{Avg. running time (s)} & \textbf{Min} & \textbf{Max} & \multicolumn{2}{c|}{\textbf{DISPONTE Time}} & \multicolumn{2}{c}{\textbf{Repair Time}} \\
		\textbf{KB}                 &  \textbf{$\pm$ std. dev. (s)}  & \textbf{(s)} & \textbf{(s)} & \textbf{\%} &         \textbf{(s)}          & \textbf{(\%)} &       \textbf{(s)}       \\ \hline\hline
		\verb|u1conf0| no prob. axioms        &      290.374 $\pm$ 888.12      &    1.176     &  2,462.829   &   0.0000    &            0.0000             &     0.000     &          0.000           \\
		\verb|u1conf0| only assertions prob.     &      317.180 $\pm$ 986.38      &    1.349     &  2,850.410   &   0.0001    &            0.0003             &     0.000     &          0.000           \\
		\verb|u1conf0| all axioms prob.       &      210.798 $\pm$ 675.07      &    1.406     &  1,976.720   &   0.0002    &            0.0004             &     0.000     &          0.000           \\ \hline
		incons. \verb|u1conf0| no prob. axioms    &     471.298 $\pm$ 1209.23      &    59.634    &  3,318.127   &   0.0000    &            0.0000             &    0.00005    &          0.0002          \\
		incons. \verb|u1conf0| only assertion prob. &     512.061 $\pm$ 1334.60      &    67.496    &  3,844.189   &   0.00004   &            0.0002             &    0.00003    &          0.0001          \\
		incons. \verb|u1conf0| all axioms prob.   &      352.537 $\pm$ 964.41      &    35.726    &  2,784.325   &   0.00010   &            0.0003             &    0.00006    &          0.0002          \\ \hline
	\end{tabular}
\end{table}

\paragraph{Scenario (SC3.3).} We modified the \verb|u1conf1| KB by associating a probability to all assertions in the KB. We imposed an increasing limit on the number of justifications to find. We ran the same 200 queries with at least one justification built for \textbf{Q2} and we collected the average running time for limit of the number of justifications with values starting from 5 and increasing of 5 each test. Table \ref{tab:bundle-inc-max} shows the results obtained in this scenario. As it can be seen, the running time quickly increases  with the limit on the number of justifications to find. With limit equals to 75 the running time became larger than the 12 hours timeout. However, the DISPONTE and Repair time is a small percentage of the whole running time.

\begin{table}
	\centering
	\caption{Average running time in seconds needed by \textit{BUNDLEInc} for computing the DISPONTE probability and the repair semantics w.r.t. \texttt{u1conf1} when increasing the limit of justifications to find. Values ``--'' mean timeout (set to 12 hours).\label{tab:bundle-inc-max}}	
	\tiny
	\begin{tabular}{c|c|c|c|c|c}
		\textbf{\# Just.} & \textbf{Time} & \multicolumn{2}{c|}{\textbf{DISPONTE Time}} & \multicolumn{2}{c}{\textbf{Repair Time}} \\
		\textbf{limit}   & \textbf{(s)}  & \textbf{\%} &         \textbf{(s)}          & \textbf{\%} &        \textbf{(s)}        \\ \hline\hline
		5         &    218.290    &   0.0060    &             0.013             &   0.0087    &           0.019            \\
		10         &    624.621    &   0.0029    &             0.018             &   0.0197    &           0.123            \\
		15         &   1,069.033   &   0.0022    &             0.023             &   0.0210    &           0.225            \\
		20         &   1,485.274   &   0.0020    &             0.030             &   0.0222    &           0.329            \\
		25         &   2,209.643   &   0.0015    &             0.034             &   0.0197    &           0.436            \\
		30         &   2,787.249   &   0.0014    &             0.040             &   0.0195    &           0.543            \\
		35         &   3,273.944   &   0.0014    &             0.045             &   0.0197    &           0.645            \\
		40         &   4,367.650   &   0.0012    &             0.051             &   0.0170    &           0.744            \\
		45         &   5,614.045   &   0.0010    &             0.055             &   0.0152    &           0.852            \\
		50         &   7,661.127   &   0.0008    &             0.060             &   0.0125    &           0.956            \\
		55         &  10,203.235   &   0.0006    &             0.066             &   0.0104    &           1.063            \\
		60         &  14,851.796   &   0.0005    &             0.072             &   0.0078    &           1.165            \\
		65         &  21,751.315   &   0.0004    &             0.078             &   0.0058    &           1.267            \\
		70         &  40,451.613   &   0.0002    &             0.083             &   0.0034    &           1.373            \\
		75         &      --       &     --      &              --               &     --      &             --             \\ \hline
	\end{tabular}
\end{table}

\paragraph{Discussions.} The results obtained show that the time to compute the DISPONTE probability and the repair semantics slightly impact the running time, while the main limitation to scalability is the search for justifications. These results confirm what we have found for question \textbf{Q2}.

\subsection{Overall Discussions}
From these results we can observe that our approach suffers of two main problems. First, the search for the justifications may be unfeasible w.r.t. some KBs. This is due to the tableau method  used to collect them. However, the search method can be replaced without affecting the computation of the DISPONTE probability and the repairs semantics. On the other hand, the computation of the DISPONTE probability is always extremely fast, independently from the number of probabilistic axioms included in the KB. Second, the time to compute the repair semantics depends more strongly from the number of probabilistic axioms and the number of justifications and becomes very high when all axioms of the KB are probabilistic and the number of different justifications is high, as shown by setting \textbf{(S3)} of scenario \textbf{(SC1.2)} and by scenario \textbf{(SC1.3)} of \textbf{Q1}. However, these last cases have 2048 justifications, a number difficult to find even in larger KBs.

\section{Conclusions}
\label{sec:conc}
We have presented a simple  approach to cope with inconsistent KBs, which does not require  changing the syntax of the logic and exploits the probabilistic semantics DISPONTE. 
Inference algorithms that use the tableau and can build the set of all the justifications of a query could be easily adapted to the new task.  This allows the application of this approach to every DL language equipped with a suitable set of tableau expansion rules. We implemented the proposed extension in TRILL, developing \textit{TRILLInc}, and BUNDLE, developing \textit{BUNDLEInc}, and tested them w.r.t. two different KBs. For the future, we plan to study the generalization to FOL of the presented extensions of the semantics and possible optimizations to better cope with the repair semantics.

\section*{Acknowledgment}
Research funded by the Italian Ministry of University and Research through PNRR - M4C2 - Investimento 1.3 (Decreto Direttoriale MUR n. 341 del 15/03/2022), Partenariato Esteso PE00000013 - ``FAIR - Future Artificial Intelligence Research'' - Spoke 8 ``Pervasive AI'', funded by the European Union under the ``NextGeneration EU programme''. This work was also partly supported by the ``National Group of Computing Science (GNCS-INDAM)'' and by the Spoke 1 ``FutureHPC \&
BigData'' of the Italian Research Center on High-Performance Computing,
Big Data and Quantum Computing (ICSC) funded by MUR Missione 4 - Next
Generation EU (NGEU).

\bibliographystyle{alphaurl}
\bibliography{bibl}

\end{document}